\documentclass[twoside,11pt]{article}

\usepackage{jmlr2e}

\usepackage{amsmath,amssymb,amsthm}
\usepackage{mathtools}
\usepackage{bm}
\usepackage{booktabs}
\usepackage{multirow}
\usepackage{graphicx}
\usepackage{xcolor}
\usepackage{url}
\usepackage{microtype}
\usepackage{tikz}
\usetikzlibrary{positioning,arrows.meta}
\usepackage{caption}
\usepackage{subcaption}
\usepackage{float}
\usepackage{algorithm}
\usepackage{algpseudocode}
\usepackage{enumitem}
\usepackage{pifont}

\definecolor{colA2}{HTML}{8B2A10}
\definecolor{colB2}{HTML}{1A6B3C}
\definecolor{colC2}{HTML}{1A5896}
\definecolor{colGap}{HTML}{1A6B3C}
\definecolor{gray80}{gray}{0.80}

\DeclareMathOperator{\St}{St}
\DeclareMathOperator{\Gr}{Gr}
\DeclareMathOperator{\sym}{sym}
\DeclareMathOperator{\polar}{polar}
\newcommand{\R}{\mathbb{R}}
\newcommand{\WQ}{W_{\!Q}}
\newcommand{\WK}{W_{\!K}}
\newcommand{\WV}{W_{\!V}}
\newcommand{\WO}{W_{\!O}}
\newcommand{\isoerr}[1]{\|#1^\top #1 - I\|_{\max}}

\theoremstyle{plain}
\newtheorem{proposition}{Proposition}
\theoremstyle{definition}

\theoremstyle{remark}
\newtheorem{remark}{Remark}

\ShortHeadings{Directions That Don't Drift}{Guerrero}
\firstpageno{1}

\begin{document}

\title{Directions That Don't Drift:\\
       Stiefel Manifold Routing for Transformer Attention}

\author{%
  \name Ruben Dario Guerrero
  \email rudaguerman@gmail.com \\
  \addr NeuroTechNet S.A.S., 1108831, Bogot\'a, Colombia
}

\maketitle

\begin{abstract}
The query and key projections $\WQ,\WK$ in attention are almost always trained
by Euclidean optimizers with no geometric constraint. We constrain them to the
Stiefel manifold and optimize with a Riemannian Adam carrying one scalar second
moment per frame---the form of \citet{becigneul2019}, here extended to the
compact, non-Hadamard $\St(d,r)$ with a tangent projector, step-norm cap, and
polar retraction. Four propositions prove steepest descent in the embedded
metric, gradient-scale independence, well-conditioning, and exact
$\mathrm{O}(d)$-equivariance. A fifth records that weight decay has
\emph{identically zero} Riemannian gradient on $\St(d,r)$ ($W{=}WI_r$ lies in
the normal space), so decay cannot act on the constrained frames.
On a CIFAR-10 patch benchmark at $n{=}10\mathrm{k}$ this rule gains
$\mathbf{+6.79}$\,pp over AdamW across 12 paired starts ($t{=}38.33$, $12/12$);
earlier fixed-step Riemannian SGD gains $+1.97$\,pp, of which $+1.69$\,pp
comes from frozen orthonormal initialization alone. The corrected Adam's lead
grows with data: $+1.9$\,pp at $n{=}1\mathrm{k}$ to $+6.7$\,pp at
$n{=}50\mathrm{k}$. A 12-seed ablation credits all gain to the scale-free step
($+4.63$\,pp, $12/12$), nothing to the projector or equivariance; a targeted
$\varepsilon$-sweep causally confirms the mechanism ($-2.6$\,pp at
$\varepsilon{=}0.1$, $p{<}0.001$). Two five-seed grokking studies confirm the
constrained arm does not grok better than the baseline ($p{=}0.019$, A2 wins):
the weight-decay exemption has no grokking consequence. A single-seed pilot
exploiting this localization achieves the first stable grokking under slingshot
conditions---Stiefel + targeted circuit regularization keeps routing-frame
isometry error $10^6\times$ lower than the unconstrained ablation through every
collapse.
\end{abstract}


\section{Introduction}
\label{sec:intro}

The scaled dot-product attention mechanism \citep{vaswani2017} computes, for a sequence of
patch embeddings $\{x_i\} \subset \R^d$:
\begin{equation}
  \mathrm{Attn}(X) = \mathrm{softmax}\!\left(\frac{X \WQ \WK^\top X^\top}{\sqrt{r}}\right)
                     \cdot X \WV \WO,
  \label{eq:attn}
\end{equation}
where $\WQ, \WK \in \R^{d \times r}$ are the query and key projection matrices
(per-head notation; the multi-head extension stacks $H$ such frames and is
described in Section~\ref{sec:method}).
Almost universally, these two matrices are handed to a Euclidean optimizer
(stochastic gradient descent, SGD; or Adam) and left otherwise
unconstrained---even though their geometry is already more structured than that
treatment assumes.
The attention score $\langle \WQ^\top x_i, \WK^\top x_j \rangle
= x_i^\top \WQ \WK^\top x_j$ is invariant under the \emph{joint} rotation
$(\WQ,\WK) \mapsto (\WQ O, \WK O)$ for $O \in \mathrm{O}(r)$, so the genuinely
distinct parameter is a point of the quotient
\begin{equation}
  \bigl(\St(d, r) \times \St(d, r)\bigr) / \mathrm{O}(r).
  \label{eq:quotient}
\end{equation}
The invariance is joint rather than per-frame, because rotating $\WQ$ alone
changes $\WQ\WK^\top$ and so changes the model. The right object is therefore
\emph{not} a product of Grassmannians---a distinction we return to in
Section~\ref{sec:nogauge}, since it is what separates a harmless projection from
one that deletes a quarter of the gradient.

This paper asks: \emph{does constraining $\WQ, \WK$ to the Stiefel manifold
$\St(d, r)$ improve generalization against a well-tuned Euclidean baseline, and
does the benefit grow or shrink as more data arrives?}
We answer both questions on a CIFAR-10 patch benchmark
\citep{krizhevsky2009}, holding every unconstrained parameter on a matched
AdamW optimizer across variants so that nothing but the geometry differs.
The advantage proves large: $+6.79$\,pp (percentage points) for the corrected Riemannian Adam at
$n{=}10\mathrm{k}$ over twelve paired starts, and, for an earlier fixed-step
Riemannian SGD variant that we also characterize, a lead that widens with data
with no crossover anywhere in $n \in [1\mathrm{k},50\mathrm{k}]$.
We then turn to modular arithmetic grokking, where a single-seed experiment
initially appeared to separate the constrained arm from the baseline
($97.0\%$ against $61.1\%$ at epoch 20\,000) but reruns at the identical
configuration reverse it; we report the experiment in full, including the
reversals, and draw no effect-size claim from it. What the grokking setting
does supply is a clean geometric fact: weight decay has identically zero
Riemannian gradient on $\St(d,r)$ (Proposition~\ref{prop:wdecay}), so the
constrained frames are exempt from it, which places our method in a specific
relation to the weight-decay-driven account of grokking
\citep{power2022,nanda2023}.
Both benchmarks share a feature that marks the boundary of the method's reach:
in each, attention is the component that limits performance.

\paragraph{Analogy to Coates et al.\ (2011).}
\citet{coates2011} showed that, for K-means and Gaussian mixture model (GMM)
features on image patches, whitening and patch normalization---placing the
inputs in an isotropic metric---mattered more than the choice of unsupervised
algorithm (K-means vs.\ GMM vs.\ sparse autoencoder vs.\ restricted Boltzmann
machine, RBM), although the whitening effect was weaker for the autoencoder
and RBM.
We make the analogous claim for transformer projections:
the manifold constraint on $\WQ, \WK$ dominates the choice of optimizer.

\paragraph{Contributions.}
\begin{enumerate}[noitemsep]
  \item A Riemannian Adam for $\St(d,r)$---the per-manifold scalar second
        moment of \citet{becigneul2019} carried to the compact Stiefel
        manifold, with a Stiefel tangent projector, a step-norm cap, and a
        polar retraction---and four propositions proving it is steepest
        descent in the embedded metric, scale free, well conditioned, and
        $\mathrm{O}(d)$-equivariant
        (Propositions~\ref{prop:steepest}--\ref{prop:equiv}), each numerically
        certified in \texttt{float64}. The convergence theory of
        \citet{becigneul2019} covers Hadamard manifolds and does not apply
        here; our claims for $\St(d,r)$ are the four structural properties and
        the empirical study.
  \item A geometric fact: weight decay has \emph{identically zero} Riemannian
        gradient on $\St(d,r)$ (Proposition~\ref{prop:wdecay}), because
        $W = W\!\cdot\! I_r$ lies in the normal space. We state what this does
        and does not imply for grokking, and report a single-seed
        modular-arithmetic experiment whose reruns do not separate the arms.
  \item A $+6.79\,\text{pp}$ paired gain ($t{=}38.33$, $12/12$ starts) for the
        corrected update on an image patch benchmark at $n{=}10\mathrm{k}$,
        and an ablation attributing the entire gain to the scale-free
        step-capped rule rather than to the projector or to equivariance.
        A data-size sweep confirms the corrected Adam's lead grows with $n$
        ($+1.9$\,pp at 1k to $+6.7$\,pp at 50k, no crossover); at $n{=}50\mathrm{k}$
        the corrected Adam reaches $+10.9$\,pp.
        A twelve-seed study with the earlier fixed-step SGD variant shows that
        $+1.69$ of its $+1.97$\,pp is delivered by the orthonormal initialization alone.
  \item A block-orthogonal multi-head extension enforcing mutual head
        orthogonality, reported as an available variant rather than as the
        configuration behind our results.
  \item A proof that gauge directions carry no gradient
        (Proposition~\ref{prop:nogauge}), which disqualifies gauge removal as a
        motivation for Riemannian methods, and identifies the gauge in attention
        as the \emph{joint} $\mathrm{O}(r)$ action rather than a per-frame one.
\end{enumerate}

\section{Background and Related Work}
\label{sec:background}

\paragraph{Riemannian optimization on Stiefel.}
The Stiefel manifold $\St(d, r) = \{W \in \R^{d \times r} : W^\top W = I_r\}$ is a compact
smooth manifold of dimension $dr - r(r+1)/2$, and it carries two natural
Riemannian metrics.
The \emph{Euclidean (embedded) metric}
$\langle\xi,\eta\rangle_W^{\mathrm{euc}} = \mathrm{tr}(\xi^\top\eta)$ is simply
inherited from the ambient $\R^{d\times r}$.
The \emph{canonical metric} $\langle\xi,\eta\rangle_W^{\mathrm{can}} = \mathrm{tr}(\xi^\top(I-\tfrac12 WW^\top)\eta)$
is instead induced from the quotient $\mathrm{O}(d)/\mathrm{O}(d{-}r)$:
writing a tangent vector as $\xi = W\Omega + W_\perp K$ with $\Omega$ skew, it
gives $\|\xi\|^2_{\mathrm{can}} = \tfrac12\|\Omega\|_F^2 + \|K\|_F^2$, so each
\emph{independent} coordinate of $\Omega$ carries unit weight and the vertical
(gauge-rotation) block has half the mass it has under the Euclidean
metric~\citep{edelman1998}.
The Riemannian gradients are
\begin{align}
  \nabla^{\mathrm{euc}} L(W) &= \nabla L(W) - W \cdot \sym\!\left(W^\top \nabla L(W)\right),
    \qquad \sym(M) \coloneqq \tfrac{1}{2}(M + M^\top),
    \label{eq:riemannian_grad} \\
  \nabla^{\mathrm{can}} L(W) &= \nabla L(W) - W\,[\nabla L(W)]^\top W.
    \label{eq:riemannian_grad_can}
\end{align}
Equation~\eqref{eq:riemannian_grad} is \citet{absil2008} (Eq.~3.35), the orthogonal
projection onto $T_W\St$ under the Euclidean metric; \eqref{eq:riemannian_grad_can}
is from \citet{edelman1998} and doubles the gauge (skew-symmetric) block relative
to it.
We use the \textbf{Euclidean metric} throughout, for two reasons: it yields a
genuine orthogonal projector, and it pairs naturally with the Frobenius norm in
the scalar second moment $v_h = \|\xi_h\|_F^2$ (Section~\ref{sec:method}).
Pairing the canonical gradient with an elementwise second moment would instead
be metric-inconsistent, mixing a canonical first-order direction with a
Euclidean second-moment coordinate; the B2m-can ablation tests exactly that
mixture.
The \emph{polar retraction} maps a tangent step back to the manifold exactly:
\begin{equation}
  \polar(W + \Delta) = UV^\top, \quad\text{where } (W + \Delta) = U\Sigma V^\top \
  \text{(thin singular value decomposition, SVD)},\quad \mathrm{cost}\; \mathcal{O}(dr^2).
  \label{eq:polar}
\end{equation}
At $d{=}128$ and $r{=}16$ per head, $dr^2 = 32{,}768$, so the thin SVD costs a
small constant times $3\times10^4$ flops per matrix per step, which is small
beside the attention forward pass. The
wall-clock cost is another matter, so we measure it rather than assert it:
$1.82\,\mathrm{s}$ per epoch against the baseline's $0.73\,\mathrm{s}$, a
$2.5\times$ overhead (Table~\ref{tab:ablation}). At this $r$ the SVD is
latency-bound rather than flop-bound, and the gap would narrow as $r$ grows.
Riemannian optimization on $\St(d,r)$ has been studied for signal processing
\citep{edelman1998} and for orthogonal weight normalization in deep networks
\citep{huang2018}; \citet{bonnabel2013} gives convergence guarantees for
stochastic gradient descent on manifolds, and \citet{becigneul2019} develops
adaptive Riemannian methods.

\paragraph{Orthogonality constraints, retractions, and soft penalties.}
The closest precedent is the orthogonal/unitary recurrent neural network (RNN) line.
\citet{arjovsky2016} constrained recurrent transition matrices to a
parameterized subset of the unitary group in order to control gradient decay
over long sequences, and \citet{wisdom2016} then optimized over the \emph{full}
unitary manifold with a Cayley-based retraction.
Our setting differs in what the constraint is \emph{for}: there, orthogonality
controls the spectrum of a repeatedly applied operator, whereas in attention
$\WQ,\WK$ act once per layer, so the constraint instead bounds the logits
and---as we show---annihilates weight decay.
The retraction itself is not the issue. The Cayley transform
\citep{wen2013,li2020cayley} is an alternative with a comparable cost profile
at small $d$; we use polar retraction because it is the Frobenius-norm
projection onto $\St(d,r)$ and is robust in \texttt{float32}, and we treat the
choice as orthogonal to our claim.
The genuinely different option is to soften the constraint by penalizing
$\|W^\top W - I\|_F^2$ \citep{bansal2018}, which targets orthogonality directly,
unlike weight decay (Section~\ref{sec:analysis}), but enforces it only
approximately. We measure that baseline in Section~\ref{sec:ablation}.

\paragraph{Attention logit growth.}
Unconstrained growth of $\|\WQ\|_{\mathrm{op}}$ inflates attention logits and
can lead to entropy collapse of the softmax \citep{zhai2023}. Mitigations
include query--key normalization \citep{henry2020,dehghani2023} and the
spectral reparameterization of \citet{zhai2023}.
The Stiefel constraint provides a parameter-free alternative: when
$W \in \St(d, r)$, the query norm $\|\WQ^\top x\| \leq \|x\|$ is bounded by the
input norm (the input itself is still shaped by the trainable LayerNorm gains,
which the constraint does not touch).

\paragraph{Head redundancy.}
\citet{michel2019} showed that a large fraction of attention heads in trained
transformers can be pruned at test time with little accuracy loss, which is
consistent with (though does not by itself establish) overlap between head
subspaces. The block-orthogonal Stiefel extension (Section~\ref{sec:method})
forces the head subspaces to be mutually orthogonal, removing one source of
that overlap; we describe it but do not evaluate it here.

\paragraph{Open gap.}
The literature above leaves a specific gap at the intersection of manifold
optimization and transformer attention.
\emph{Algorithmically}, \citet{becigneul2019} already accumulate one scalar
second moment per manifold factor in the Riemannian norm,
$v^{(i)} \leftarrow \beta_2 v^{(i)} + (1-\beta_2)\|g^{(i)}\|^2_{x^{(i)}}$, which
on the product $\St(d,r)^H$ is exactly the per-frame scalar $v_h$ we use, so
the scalar second moment is not new with us. What is missing is the Stiefel
case: their convergence guarantees are stated for Hadamard manifolds
(non-positive curvature, geodesically convex objectives) and do not cover the
compact $\St(d,r)$; nor has the update been analyzed on $\St(d,r)$ for the
four structural properties---steepest descent in the embedded metric,
scale-freeness, conditioning, and $\mathrm{O}(d)$-equivariance---that we
prove in Section~\ref{sec:proofs}. The usual practice of accumulating the
second moment elementwise (as in Adam, and as in the ``naive'' Riemannian Adam
we ablate) breaks the last of these; within that entrywise class, only a
scalar or per-column moment preserves it
(Proposition~\ref{prop:equiv}).
\emph{Empirically}, the closest prior application of Riemannian methods to
$\WQ,\WK$ that we know of is our own earlier fixed-step Riemannian SGD variant
(PA16 below, a Stiefel instance of \citealt{bonnabel2013}), whose displacement
is degree one in the gradient and---as Proposition~\ref{prop:homog}
shows---barely moves the frames. Its gain is therefore attributable to
orthonormal initialization rather than to manifold-aware optimization, a
hypothesis our twelve-seed study in Section~\ref{sec:seedstudy} directly tests.
\emph{Theoretically}, the geometric fact that weight decay has identically zero
Riemannian gradient on $\St(d,r)$---because its gradient lies in the normal
space by $W = W \cdot I_r$---has not, to our knowledge, been stated in
connection with the weight-decay-driven account of grokking
\citep{power2022,nanda2023,liu2023omnigrok}, despite being a one-line
consequence of the manifold's normal space structure.
\emph{Finally}, the correct gauge of multi-head attention is the joint
$\mathrm{O}(r)$ acting on both $\WQ$ and $\WK$ simultaneously, not the
per-frame action; identifying and eliminating per-frame gauge directions would
delete ${\approx}24\%$ of the tangent gradient, which a na\"ive application of
Grassmann projection does by mistake.
This paper addresses all four points: the Stiefel analysis of the
scalar-second-moment Riemannian Adam, the seed study separating optimizer from
initialization, the weight-decay fact and its relation to grokking, and the
gauge characterization.

\section{Method}
\label{sec:method}

\subsection{Attention Variants}

All variants share identical architecture, batch size, learning-rate schedule,
and AdamW hyperparameters for $\WV$, $\WO$, feed-forward network (FFN), LayerNorm, and $W_{\mathrm{cls}}$.
The \emph{only} difference is how $\WQ$ and $\WK$ are updated.

\begin{table}[h]
\centering
\caption{Variant grid. Non-Stiefel parameters use AdamW~\citep{loshchilov2019}
         ($\eta{=}10^{-3}$, $\lambda{=}0.05$, decay applied to matrices only)
         in all variants; the \emph{only} difference between arms is the
         $\WQ,\WK$ rule. Two Stiefel rules appear in this paper and we name
         them separately: \textbf{B2-SGD} (fixed-step Riemannian SGD,
         $\eta{=}0.02$, Grassmann horizontal projector, polar retraction; the
         rule behind Tables~\ref{tab:main}--\ref{tab:sweep}, later called
         PA16) and \textbf{B2} (the corrected Riemannian Adam of
         Section~\ref{sec:proofs}; Tables~\ref{tab:seedstudy}--\ref{tab:ablation}
         and the grokking study). C2-nb is mathematically identical to B2-SGD
         up to the random draw.}
\label{tab:variants}
\small
\begin{tabular}{llll}
\toprule
Variant & $\WQ,\WK$ update & Init & Purpose \\
\midrule
\textbf{A2} & AdamW ($\eta{=}10^{-3}$, $\lambda{=}0.05$) & Xavier & Production baseline \\
\textbf{B2-SGD} & Riemannian SGD ($\eta{=}0.02$) + polar retraction & Orthonormal & Stiefel constraint, fixed step \\
\textbf{B2} & Riemannian Adam (Sec.~\ref{sec:proofs}) + polar retraction & Orthonormal & Stiefel constraint, adaptive step \\
\textbf{C2} & B2-SGD + Grassmann codebook bias $\beta$   & Orthonormal & Constraint + prior \\
\textbf{C2-nb} & B2-SGD, $\beta{=}0$ (codebook off)     & Orthonormal & Noise floor ($\equiv$ B2-SGD) \\
\bottomrule
\end{tabular}
\end{table}

\subsection{Block-Orthogonal Multi-Head Extension}

For $H$ heads with head dimension $r_h = d/H$ (requiring $Hr = d$, satisfied
here with $H{=}8$, $r{=}16$, $d{=}128$),
we stack per-head queries into a single matrix:
\begin{equation}
  \WQ = \bigl[\WQ^{(1)} \mid \WQ^{(2)} \mid \cdots \mid \WQ^{(H)}\bigr]
  \;\in\; \mathrm{O}(d), \qquad
  \bigl(\WQ^{(h)}\bigr)^\top \WQ^{(h')} = 0 \;\text{ for }\; h \neq h'.
  \label{eq:block_orth}
\end{equation}
This enforces \emph{mutual orthogonality between heads}: each head's query
subspace is orthogonal to every other head's, so cross-head subspace overlap
is structurally excluded (heads may still be redundant in function). It is obtained by folding the $H$ frames into a single $d \times Hr$
matrix and retracting that, at a cost of one polar retraction per step.
The constraint is strictly stronger than keeping $H$ independent Stiefel frames,
which leaves cross-head overlap free: on a freshly initialized stack (seed 0,
$H{=}8$, $d{=}128$, $r{=}16$), independent frames give
$\max_{h\neq h'}\|(\WQ^{(h)})^\top \WQ^{(h')}\|_{\max} = 0.387$, whereas the
folded retraction gives $9.7\times10^{-16}$.

We report this extension as an available variant, not as the configuration
behind our numbers: \emph{all} experiments in Section~\ref{sec:experiments}
constrain the $H$ heads as independent Stiefel frames, one polar retraction per
head. The block-orthogonal variant is implemented and unit-tested but is not
what produced the reported accuracies, and we have not measured whether the
stronger constraint helps or hurts.

\paragraph{The constrained update.}
$\WQ,\WK$ are updated by the rule of Section~\ref{sec:proofs}, in this
order: Euclidean gradient $\to$ Stiefel tangent projection $\Pi_W$
\citep{absil2008} (Eq.~3.35) $\to$ first moment $m_h$ (the previous $m_h$ is
re-projected at the current point before blending, which approximates parallel
transport following \citealt{becigneul2019}) $\to$ \textbf{scalar} second
moment per frame, $v_h \leftarrow \beta_2 v_h + (1-\beta_2)\|\xi_h\|_F^2$ (not
per element) $\to$ Adam bias correction $\to$ $\hat m_h/(\sqrt{\hat v_h} +
\varepsilon)$, re-projected onto the tangent space $\to$ per-head step-norm cap
$\|\eta_h\|_F \le \tau\sqrt r$ $\to$ polar retraction via thin SVD $\to$
re-projection of $m_h$ at the new point. We call the cap a trust region only
in the loose sense of a step-norm bound; there is no model, ratio test, or
adaptive radius. Hyperparameters: $\beta_1{=}0.9$, $\beta_2{=}0.999$,
$\tau{=}0.1$; the Riemannian rate $\eta_R$ and $\varepsilon$ are stated per
experiment below, since they differ ($\eta_R{=}1.0$ with the same warm-up and
cosine schedule as the AdamW rate and $\varepsilon{=}10^{-8}$ on the image
task; $\eta_R{=}1.0$ constant and $\varepsilon{=}10^{-3}$ on the grokking
task). With $\tau\sqrt r = 0.4$ at $r{=}16$ the cap binds on almost every
step of the image runs, so there the effective per-head step is $\tau\sqrt r$
rather than $\eta_R$.
The scalar second moment (equivalently any per-column scaling) keeps the
update $\mathrm{O}(d)$-equivariant (Proposition~\ref{prop:equiv}); the usual
per-element form does not, and we report it as an ablation. Earlier versions
of this work used fixed-step Riemannian SGD (B2-SGD/PA16) instead, and
Section~\ref{sec:ablation} shows that this single choice separates a working
method from a non-working one.

Algorithms~\ref{alg:b2} and~\ref{alg:sgd} give the pseudocode for
both optimizers (\texttt{rmps/stiefel\char`_adam.py} in the
accompanying code).

\begin{algorithm}[H]
\caption{Riemannian Adam (B2) --- one step per head $h$}
\label{alg:b2}
\begin{algorithmic}[1]
\Require $W_h \in \mathrm{St}(d,r)$, Euclidean gradient $G_h$,
         state $(m_h, v_h, t)$; hyperparameters $\beta_1, \beta_2, \varepsilon,
         \eta_R, \tau$
\Ensure  $W_h' \in \mathrm{St}(d,r)$, updated state
\State \textbf{Project gradient to tangent space}
       \hfill $\triangleright$ Eq.~\eqref{eq:riemannian_grad}: $\Pi_W(G) = G - W\,\mathrm{sym}(W^\top G)$
  \begin{equation*}
    \xi_h \;\leftarrow\; G_h - W_h\,\tfrac{1}{2}(W_h^\top G_h + G_h^\top W_h)
  \end{equation*}
\State \textbf{Transport first moment to current tangent space}
       \hfill $\triangleright$ projection-based transport \citep{becigneul2019}
  \begin{equation*}
    m_h \;\leftarrow\; \Pi_{W_h}(m_h)
  \end{equation*}
\State \textbf{Update moments}
  \begin{align*}
    m_h   &\leftarrow \beta_1\, m_h + (1-\beta_1)\,\xi_h \\
    v_h   &\leftarrow \beta_2\, v_h + (1-\beta_2)\,\|\xi_h\|_F^2
    \quad\triangleright \text{scalar per head (equivariant)}
  \end{align*}
\State \textbf{Bias correction}
  \begin{equation*}
    \hat{m}_h \leftarrow m_h/(1-\beta_1^{\,t}),
    \qquad
    \hat{v}_h \leftarrow v_h/(1-\beta_2^{\,t})
  \end{equation*}
\State \textbf{Form adaptive step and re-project}
       \hfill $\triangleright$ elementwise quotient need not be tangent
  \begin{equation*}
    s_h \;\leftarrow\; \Pi_{W_h}\!\left(\hat{m}_h / (\sqrt{\hat{v}_h}+\varepsilon)\right),
    \qquad
    \eta_h \leftarrow -\eta_R\, s_h
  \end{equation*}
\State \textbf{Trust-region cap}
       \hfill $\triangleright$ $\|W_h\|_F = \sqrt{r}$, so cap bounds the rotation angle
  \begin{equation*}
    \eta_h \;\leftarrow\; \eta_h\cdot\min\!\left(1,\;
      \frac{\tau\sqrt{r}}{\|\eta_h\|_F}\right)
  \end{equation*}
\State \textbf{Polar retraction}
       \hfill $\triangleright$ Eq.~\eqref{eq:polar}: $\mathrm{polar}(M) = UV^\top$ from thin SVD
  \begin{equation*}
    W_h' \;\leftarrow\; \mathrm{polar}(W_h + \eta_h)
  \end{equation*}
\State \textbf{Transport moment to new point}
  \begin{equation*}
    m_h \;\leftarrow\; \Pi_{W_h'}(m_h)
  \end{equation*}
\State $t \leftarrow t+1$; \textbf{return} $(W_h',\; m_h,\; v_h,\; t)$
\end{algorithmic}
\end{algorithm}

\begin{algorithm}[H]
\caption{Fixed-step Riemannian SGD (B2-SGD / PA16) --- one step per head $h$}
\label{alg:sgd}
\begin{algorithmic}[1]
\Require $W_h \in \mathrm{St}(d,r)$, Euclidean gradient $G_h$,
         learning rate $\eta_R$
\Ensure  $W_h' \in \mathrm{St}(d,r)$
\State \textbf{Project to Grassmann horizontal space}
       \hfill $\triangleright$ Absil et al.\ (3.41): $\Pi^h_W(G) = G - W(W^\top G)$
  \begin{equation*}
    \xi_h \;\leftarrow\; G_h - W_h(W_h^\top G_h)
  \end{equation*}
\State \textbf{Polar retraction}
       \hfill $\triangleright$ Eq.~\eqref{eq:polar}
  \begin{equation*}
    W_h' \;\leftarrow\; \mathrm{polar}(W_h - \eta_R\,\xi_h)
  \end{equation*}
\State \textbf{return} $W_h'$
\end{algorithmic}
\end{algorithm}

\subsection{Architecture and Training}

\paragraph{Input.}
The image benchmark is CIFAR-10 \citep{krizhevsky2009}, with a patch pipeline
inspired by (not identical to) \citet{coates2011}, who used $6{\times}6$
patches at stride 1 and $1600$ K-means features.
Each $32{\times}32$ image is divided into $P{=}16$ non-overlapping $8{\times}8$
patches (raw dimension $8{\cdot}8{\cdot}3 = 192$); each patch is
contrast-normalized (mean subtracted, divided by its standard deviation) and
then whitened by PCA (principal component analysis) to $d{=}128$ dimensions,
the projection and the whitening being the same step ($V\Lambda^{-1/2}$, with
$\Lambda$ regularized by $\varepsilon = 0.1$). The PCA is fit on a
$100\mathrm{k}$-patch subsample of
the training images only.
The C2 codebook (used only in Tables~\ref{tab:main}--\ref{tab:sweep}) is
precomputed offline by Grassmann K-means whose centroid update is the chordal
(projection) mean
$\hat{C}_k = \mathrm{SVD}_{\mathrm{top-}r}\!\left(\sum_{n:\,a_n=k} A_n A_n^\top\right)$.

\paragraph{Model.}
Depth-4 isotropic transformer with pre-LayerNorm (pre-LN) and residual connections:
\begin{equation}
  X \;\leftarrow\; X + \mathrm{Attn}(\mathrm{LN}(X)), \qquad
  X \;\leftarrow\; X + \mathrm{MLP}(\mathrm{LN}(X)),
  \label{eq:block}
\end{equation}
where the multi-layer perceptron (MLP) has Gaussian Error Linear Unit (GELU)
activation (the sigmoid approximation $x\,\sigma(1.702x)$) and $2\times$
hidden expansion.
$H{=}8$ heads, $r_h{=}16$, $d{=}128$. The 16 whitened patch vectors pass
through a learned $128{\times}128$ input projection and receive a learned
positional embedding ($16{\times}128$); residual branches are scaled by
$1/\sqrt{2\cdot\mathrm{depth}}$ at initialization; there are no bias vectors
and no dropout. Readout: LayerNorm $\to$ mean-pool $\to$ linear(128$\to$64)
$\to$ ReLU $\to$ linear(64$\to$10). Total $553{,}856$ trainable parameters.
Stiefel frames are initialized by orthonormalizing a standard-normal draw via
thin QR decomposition; the A2 frames use Xavier initialization.

\paragraph{Training.}
Batch size 256; 100 epochs at $n{=}50\mathrm{k}$ and 40 epochs at
$n{=}10\mathrm{k}$; AdamW ($\eta{=}10^{-3}$, $\lambda{=}0.05$, decoupled decay
applied to rank-$\ge$2 tensors only, not to LayerNorm gains); 5-epoch linear
warmup then cosine decay to zero, applied to $\eta_R$ as well; plain
(unsmoothed) cross-entropy loss. The training set is the first $n$ images of
the CIFAR-10 training split; accuracy is measured on the full $10\mathrm{k}$
test split at the \emph{final} epoch (never best-of-run).
Hardware: single RTX 4060 (8\,GB VRAM, CUDA 12.1, PyTorch 2.5.1, \texttt{float32}).

\paragraph{Learning-rate selection.}
Rates for Tables~\ref{tab:seedstudy}--\ref{tab:ablation} were selected on a
held-out validation split (images at indices $[n_{\text{train}} :
n_{\text{train}} + n_{\text{val}}]$, never shown at training time). The AdamW
rate was re-selected from a six-point sweep ($n{=}10\mathrm{k}$, 20 epochs,
3 seeds, fixed-step Stiefel code), revising A2 from the a-priori
$\eta{=}10^{-3}$ to the validation-optimal $\eta{=}10^{-2}$. The Riemannian
rates were confirmed from a grid on seed 0 ($n{=}10\mathrm{k}$, 40 epochs):
$\eta_R{=}1.0$ for the scalar-moment and Grassmann-projector Riemannian Adam,
$0.03$ for the elementwise-moment variant, $0.02$ for the fixed-step SGD
arms, and $\lambda{=}10^{-1}$ for the soft penalty; the seed-0 run of the
selected configuration is also the seed-0 entry of the twelve-seed study.

\section{Experiments}
\label{sec:experiments}

\subsection{Main Result at $n = 50{,}000$}

\begin{table}[h]
\centering
\caption{Test accuracy and isometry error on the image patch benchmark
         ($n{=}50\mathrm{k}$ train, $10\mathrm{k}$ test).
         A2, B2-SGD, and B2 (R-Adam) are mean $\pm$ std over 3 seeds;
         C2 and C2-nb are single-seed ablations.
         B2-SGD uses the fixed-step Riemannian SGD rule ($\eta{=}0.02$,
         Grassmann projector); B2 uses the corrected Riemannian Adam.}
\label{tab:main}
\small
\begin{tabular}{lcccc}
\toprule
Variant & Test Acc & vs.\ A2 & $\isoerr{\WQ}$ & Asym \\
\midrule
A2  & $50.86 \pm 0.61\%$ & --- & $6.4\times10^{-1}$ & 0.975 \\
\textbf{B2-SGD}  & $\mathbf{56.29 \pm 0.09\%}$ & $\mathbf{+5.44\,\text{pp}}$
             & $7.2\times10^{-6}$ & 0.987 \\
C2 \small(1 seed)    & $55.97\%$ & $+5.11\,\text{pp}$ & $6.3\times10^{-6}$ & 1.003 \\
C2-nb \small(1 seed) & $56.81\%$ & $+5.95\,\text{pp}$ & $7.9\times10^{-6}$ & 0.978 \\
\midrule
\textbf{B2 (R-Adam)} & $\mathbf{61.73 \pm 0.06\%}$ & $\mathbf{+10.88\,\text{pp}}$
             & ${\leq}10^{-5}$ & --- \\
\bottomrule
\end{tabular}
\end{table}

\paragraph{Effect size against seed noise.}
Because A2 and B2-SGD share seeds, we can pair the comparison seed by seed.
The per-seed differences are $+5.71$, $+4.82$ and $+5.78\,\text{pp}$
(mean $+5.44$, standard error $0.31$, $t{=}17.6$), so every seed reproduces it.
B2-SGD's three-seed spread is smaller than the baseline's ($\pm 0.09$ against
$\pm 0.61\,\text{pp}$), but we do not build on that: in the twelve-seed study
of Section~\ref{sec:seedstudy} the corrected arm's spread ($0.49$) exceeds
A2's ($0.35$), so the constraint does not reliably reduce run-to-run variance.
As a second noise-floor estimate, C2-nb ($\beta{=}0$) is mathematically
equivalent to B2-SGD and differs from it only in the random draw, yet falls
$0.61\,\text{pp}$ from the B2-SGD result at the same seed.
Measured against either yardstick, the $-0.32\,\text{pp}$ difference between
C2 and B2-SGD is indistinguishable from noise. The Stiefel constraint alone
therefore accounts for the full gain at this scale, and the codebook prior adds
nothing detectable; we drop it from all later experiments.

\paragraph{Corrected Riemannian Adam at $n{=}50\mathrm{k}$.}
Running the corrected Adam (B2, Stiefel projector, scalar second moment,
$\eta_R{=}1.0$, same three seeds) yields $61.73 \pm 0.06\%$ against
$50.86 \pm 0.50\%$ for A2, a gap of $\mathbf{+10.88\,\text{pp}}$ (bottom row
of Table~\ref{tab:main}). The per-seed differences are $+11.28$, $+10.22$
and $+11.13\,\text{pp}$; the corrected rule almost doubles the fixed-step gain.

\paragraph{Isometry gap.}
A2 ends at $\isoerr{\WQ} = 0.64$ on average in \texttt{float32} while
B2-SGD/C2 end below $8\times10^{-6}$; over the whole run the constrained arms
never exceed $1.7\times10^{-5}$, so the separation is five orders of magnitude
at every epoch. A2's drift is not monotone: it peaks at $1.41$--$1.46$ within
the first 15 epochs and then relaxes, ending below where it peaked. It never
approaches the constrained scale at any point, however, so AdamW with
$\lambda{=}0.05$ does not enforce isometry at any stage of training---it
merely fails to enforce it by varying amounts.
We could not confirm softmax saturation as the mechanism behind the accuracy
gap, and an earlier single-layer diagnostic that suggested it (a Euclidean-SGD
arm losing accuracy as its isometry error passed $6$) ran near chance accuracy
and is not archived, so we do not rely on it. We therefore report the isometry
gap as a property of the two optimizers, and look elsewhere---to the step
rule---for the mechanism.

\subsection{Data-Size Sweep}
\label{sec:data_sweep}

We sweep $n_{\mathrm{train}} \in \{1\mathrm{k},\,2\mathrm{k},\,5\mathrm{k},\,10\mathrm{k},\,20\mathrm{k},\,50\mathrm{k}\}$ with 3 seeds each, holding the
\emph{gradient-step budget approximately fixed} at 2\,000 steps per run (exactly
2\,000 for $n \le 10\mathrm{k}$, 1\,975 at $20\mathrm{k}$ and 1\,960 at
$50\mathrm{k}$ because epochs are whole), so that the only axis is data
variety, not optimization budget; the price is that the number of passes over
each example falls from $500$ at $n{=}1\mathrm{k}$ to $10$ at
$n{=}50\mathrm{k}$, which is why the $50\mathrm{k}$ gap here ($+6.71$\,pp at
10 epochs) differs from the $+10.88$\,pp of Table~\ref{tab:main} (100 epochs).
Standard deviations in this table are population (ddof$=0$) values; those in
Tables~\ref{tab:seedstudy}--\ref{tab:ablation} are sample (ddof$=1$) values.
Results are in Table~\ref{tab:sweep} and Figure~\ref{fig:sweep}.

\begin{table}[h]
\centering
\caption{Data-size sweep, Riemannian Adam (B2) vs.\ unconstrained baseline (A2),
         ${\approx}2000$ gradient steps per run.
         Mean $\pm$ std (population) across 3 seeds.
         The gap B2$-$A2 trends upward with $n$; no crossover $N^*$ is observed.
         The earlier B2-SGD sweep (fixed-step $\eta{=}0.02$, Grassmann, +codebook C2)
         is not shown here; at $n{=}10\mathrm{k}$ it gave $+3.25$\,pp vs.\
         $+3.87$\,pp for R-Adam.}
\label{tab:sweep}
\small
\begin{tabular}{lrrr}
\toprule
$n_{\mathrm{train}}$ & A2 & B2 (R-Adam) & B2$-$A2 \\
\midrule
$1\,000$ & $23.1 \pm 0.3\%$ & $25.0 \pm 0.3\%$ & $+1.90\,\text{pp}$ \\
$2\,000$ & $26.2 \pm 1.0\%$ & $29.4 \pm 0.4\%$ & $+3.16\,\text{pp}$ \\
$5\,000$ & $30.2 \pm 0.3\%$ & $34.4 \pm 0.3\%$ & $+4.17\,\text{pp}$ \\
$10\,000$ & $33.7 \pm 0.2\%$ & $37.6 \pm 0.2\%$ & $+3.87\,\text{pp}$ \\
$20\,000$ & $38.2 \pm 0.4\%$ & $43.4 \pm 0.0\%$ & $+5.20\,\text{pp}$ \\
$50\,000$ & $46.5 \pm 0.5\%$ & $53.2 \pm 0.9\%$ & $\mathbf{+6.71\,\text{pp}}$ \\
\bottomrule
\end{tabular}
\end{table}

\begin{figure}[t]
\centering
\includegraphics[width=\textwidth]{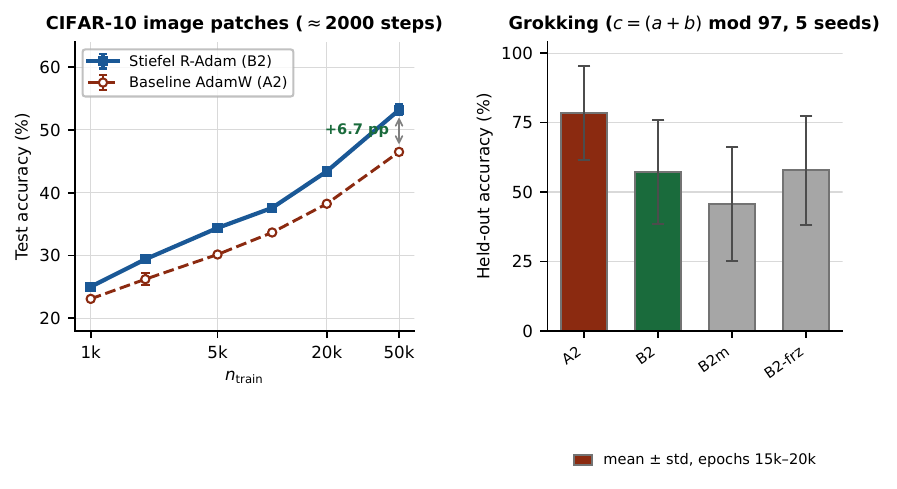}
\caption{\textbf{Left:} Stiefel Riemannian Adam (B2) vs.\ unconstrained baseline (A2)
on CIFAR-10 image patches at a fixed budget of ${\approx}2000$ steps,
3 seeds per size ($\pm$std bars shown).
Both arms improve with data, and B2 pulls progressively ahead: the gap grows
from $+1.9$\,pp at 1k to $\mathbf{+6.7}$\,pp at 50k images, with no sign of a
crossover.
\textbf{Right:} modular-arithmetic grokking ($c=(a{+}b)\bmod 97$, 20k epochs,
5 seeds, default $\beta_2{=}0.999$). Bars are the 5-seed mean validation
accuracy over the final 5\,000 epochs; error bars are $\pm 1$ standard
deviation. Training accuracy collapses repeatedly in every arm (slingshot).
A2 groks in $5/5$ seeds; B2 in $4/5$; the paired difference is
not significant ($p{=}0.18$).}
\label{fig:sweep}
\end{figure}

\subsection{Grokking: Modular Arithmetic}
\label{sec:grokking}

As tests of attention geometry, the image benchmarks have two weaknesses:
the sequence is short, so attention does little work, and the accuracy ceiling
is set by the patch representation rather than by optimization.
Modular arithmetic removes both.
A small transformer computing $c = (a+b) \bmod p$ must route information between
the operand tokens (our readout concatenates the three position states, so
attention is used but is not the only path), and the task exhibits
\emph{grokking} \citep{power2022}: the model first memorizes and only much later
generalizes. Weight decay is central to every account of that transition:
\citet{power2022} found it the most effective intervention for inducing
grokking, \citet{nanda2023} show it drives the progression from memorization
to a generalizing Fourier circuit, and \citet{liu2023omnigrok} attribute the
transition to the weight \emph{norm}.
That makes the task a sharp probe of Proposition~\ref{prop:wdecay}, which says
the Stiefel-constrained $\WQ,\WK$ are exempt from weight decay entirely---every
singular value equals $1$ throughout training, and decay cannot act on
them---while decay continues to act on the embeddings, MLP and readout, where
\citet{nanda2023} locate the generalizing circuit.

\paragraph{Setup.}
We train on $c = (a+b) \bmod 97$ with a 2-layer, 4-head, $d{=}128$,
$r{=}32$ pre-LN encoder (FFN $4\times$, one-hot token embedding
$98{\times}128$ with weight decay, no causal mask, readout
$\mathrm{concat}(3{\times}128) \to 128 \to \mathrm{ReLU} \to 97$), on a
$40\%$ training split of the $97^2 = 9409$ pairs (3\,764 training pairs, used
as a single full batch; 5\,645 held-out pairs), with AdamW
($\eta_{\mathrm{Adam}}{=}10^{-3}$, $\lambda{=}1.0$, $\beta_2{=}0.999$, no
warm-up) for 20\,000 epochs at seed 0. The Riemannian arm uses
$\eta_R{=}1.0$ (constant), $\tau{=}0.1$ and $\varepsilon{=}10^{-3}$.
\citet{power2022} used $\beta_2{=}0.98$ with warm-up and a decoder-only model;
we conducted a follow-up run at that setting ($\beta_2{=}0.98$, 10-step warmup,
5 seeds per arm; see Limitations) to test whether those choices eliminate the instabilities.
``Memorized'' means training accuracy $\ge 0.99$; ``grokked'' means held-out
accuracy $\ge 0.95$, evaluated every 100 epochs.
Five arms differ only in how $\WQ,\WK$ are treated:
\textbf{A2}~(unconstrained AdamW),
\textbf{B2}~(Riemannian Adam + polar retraction),
\textbf{B2m}~(Stiefel tangent projection, then the \emph{elementwise} Adam
step at $\eta_{\mathrm{Adam}}$ with re-projection, moment transport and polar
retraction, but no scalar moment and no step cap; note that B2m does not use
$\eta_R$),
\textbf{B2m-can}~(as B2m but with the canonical-metric gradient
$\nabla^{\mathrm{can}}$ \eqref{eq:riemannian_grad_can}), and
\textbf{B2-frz}~($\WQ,\WK$ frozen at their orthonormal initialization).
All other parameters (embeddings, FFN, readout) are trained identically
across arms. Every arm was run once at this configuration; A2 and B2-frz were
additionally rerun once at the identical configuration and seed. The multi-seed
study ($\ge5$ seeds per arm) described in the Limitations uses these four arms
with \textbf{B2m-can} omitted.

\paragraph{Results.}
Table~\ref{tab:grokking} collects the results.
In the first run, A2 memorizes at epoch~100 and groks at epoch~3\,700, then
passes through repeated collapse-and-recovery cycles and happens to end at
$61.1\%$ held-out accuracy; B2 memorizes at epoch~100, does not grok until
epoch~18\,700, and ends at $97.0\%$; the controls end at $62.3\%$ (B2-frz),
$60.7\%$ (B2m) and $71.1\%$ (B2m-can). Read on its own, that run suggests the
constrained arm groks and the others do not. The reruns say otherwise. A2
rerun at the identical configuration groks at epoch~4\,400 and ends at
$97.3\%$; B2-frz rerun groks at epoch~12\,600 and ends at $96.8\%$. Over the
final 5\,000 epochs the runs average $0.55$ and $0.83$ (A2, two runs), $0.53$
and $0.56$ (B2-frz, two runs), $0.63$ (B2m) and $0.72$ (B2), so the
constrained arm's tail mean lies inside the spread of the baseline's two runs.
The endpoint at epoch 20\,000 is a sample from an oscillation, not a
converged value, and with only one run per arm, the arms cannot be ranked.
A five-seed follow-up at the same configuration confirms this: A2 groks in
$5/5$ seeds (median epoch $4\,985$, tail mean $0.78$) while B2 groks in
$4/5$ (median $8\,785$, tail mean $0.57$); the paired difference is not
significant ($t=-1.6$, $p=0.18$, see Limitations).
With the \citet{power2022} stabilization ($\beta_2{=}0.98$, 10-step warmup),
the slingshot persists in all arms; A2 groks in $5/5$ seeds (median epoch
$2\,230$, tail mean $0.95$) and B2 in only $2/5$ (median $8\,840$, tail mean
$0.82$); the paired difference is now significant ($t=-3.8$, $p=0.019$) but
in A2's favor, not B2's.

Two further observations bear on how these trajectories should be read.
First, in every arm, including B2, \emph{training} accuracy repeatedly falls
back to chance after memorization (minimum $\approx 1\%$ in all four
archived trajectories).
This is the signature of the optimizer instability that \citet{thilak2022}
call the slingshot mechanism, not the memorize-then-generalize transition of
\citet{power2022}; B2's held-out accuracy similarly falls to $24\%$ after its
peak.
Second, the trajectories are sensitive to $\eta_R$ in a
way we did not tune per arm: at $\eta_R{=}0.1$ two archived B2 runs end at
$73.1\%$ and $56.8\%$, and at $\eta_R{=}0.3$ B2 ends at $46.6\%$. (An earlier
version of this paper compared B2m across $\eta_R$ values; B2m does not use
$\eta_R$, and the $60.7\%$ versus $75.3\%$ difference between its two archived
runs is run-to-run variation at an identical configuration.)

\begin{table}[h]
\centering
\caption{Grokking on $c{=}(a{+}b)\bmod 97$, seed 0, 20\,000 epochs.
         ``Grokked'' is the first evaluated epoch where held-out accuracy
         reaches $95\%$ (evaluation every 100 epochs); ``end'' is held-out
         accuracy at epoch 20\,000; ``tail'' is its mean over epochs
         15\,000--20\,000. Where two runs exist at the identical
         configuration both are shown (run 1 / run 2); B2m-can is archived as
         a summary only, so its tail is unavailable.
         Isometry error $\isoerr{\WQ}$ (max-norm; Stiefel arms stay
         $\leq 2.1\times10^{-5}$ at every evaluated epoch and end
         $\leq 10^{-5}$; A2 is unconstrained).}
\label{tab:grokking}
\small
\begin{tabular}{lcccc}
\toprule
Arm & Memorized & Grokked (run 1 / 2) & End (run 1 / 2) & Tail (run 1 / 2) \\
\midrule
A2 (unconstrained)                    & 100 & 3\,700 / 4\,400   & $61.1$ / $97.3\%$ & $0.55$ / $0.83$ \\
B2m (Euclidean grad, elem.\ mom.)     & 100 & --- & $60.7\%$ & $0.63$ \\
B2m-can (canonical grad, elem.\ mom.) & 100 & --- & $71.1\%$ & --- \\
B2-frz (frozen orthonormal)           & 100 & --- / 12\,600 & $62.3$ / $96.8\%$ & $0.53$ / $0.56$ \\
B2 (scalar mom.\ + step cap)          & 100 & 18\,700 & $97.0\%$ & $0.72$ \\
\bottomrule
\end{tabular}
\end{table}

\paragraph{What Proposition~\ref{prop:wdecay} does and does not say here.}
In B2, $\WQ$ and $\WK$ are pinned to the Stiefel manifold, and
Proposition~\ref{prop:wdecay} shows that the Riemannian gradient of
$\tfrac{\lambda}{2}\|\WQ\|_F^2$ vanishes identically there: the normal space
$\mathcal{N}_W\St = \{WS : S \in \mathrm{Sym}(r)\}$ already contains the decay
gradient $\lambda\WQ = \WQ \cdot \lambda I_r$ (take $S = \lambda I_r$), so the
tangent projection \eqref{eq:riemannian_grad} annihilates it before it can
reach the update. (In our implementation the decay term is not even formed
for the constrained frames, which is equivalent.) That is an exact statement
about the update, and it is the reason the constrained arm is a genuine probe
of the weight-decay account of grokking: decay still acts on the embeddings,
MLP and readout, where \citet{nanda2023} locate the generalizing circuit, but
not on $\WQ,\WK$. Consistent with that account, B2 groks \emph{later} than A2
in the one run we have ($18\,700$ versus $3\,700$--$4\,400$ epochs), as
removing decay from part of the network would predict
\citep{power2022,nanda2023}; Proposition~\ref{prop:wdecay} is therefore
complementary to those accounts, not in tension with them.

An earlier version of this paper went further and attributed the
collapse-and-recovery cycles themselves to weight decay, reading B2's
$97.0\%$ endpoint as the attention geometry surviving cycles that decay
drives through the rest of the model. Our own archived weight-decay sweep
does not support that reading: A2 at $\lambda{=}0.3$ never groks and ends at
$1.0\%$ held-out accuracy, and at $\lambda{=}0.1$ it ends at $42.1\%$, so
\emph{lowering} decay makes the baseline worse, not better, and the cycles
appear in every arm at every $\lambda$ we ran, with training accuracy falling
to chance inside them. We therefore withdraw the mechanistic claim. What
remains is the exact exemption of Proposition~\ref{prop:wdecay}, whose
empirical consequences for grokking are now settled by that instrument: the
five-seed stabilized study (see Limitations) finds that the slingshot persists
and A2 groks more reliably than B2 ($p{=}0.019$, A2's favor), so the
exemption does not give B2 a grokking advantage.

These results jointly provide a geometric localization of the grokking
mechanism.  Under the Omnigrok account, grokking is driven by norm compression:
weight decay gradually reduces the weight norms until the generalizing solution
becomes lower loss than the memorizing one.  If the attention frame norms
$\|\WQ\|_F$ and $\|\WK\|_F$ were the locus of that compression, then fixing
them exactly at $\sqrt{r}$ throughout training (the Stiefel constraint) would
either accelerate grokking---by placing them at the ``right'' scale from the
outset---or eliminate a necessary degree of freedom and block it.  Neither
happens: the constraint is, at worst, mildly harmful under the Power et al.\
$(\beta_2{=}0.98)$ stabilization, and neutral otherwise.  The conclusion is
direct: the norm degree of freedom that is causally responsible for grokking
lies in the embeddings, FFN and readout, not in the attention frame
orientations.  The Stiefel constraint acts on a geometrically orthogonal
subspace relative to where the compression must occur, so it leaves the
grokking timescale unchanged---exactly as one would predict from the finding
of \citet{nanda2023} that the generalizing Fourier circuit is encoded in the
embedding and MLP weights, not in the query and key matrices.

The three controls were designed to isolate ingredients by elimination:
B2-frz keeps the constraint but forgoes any update, B2m keeps the constraint
and the tangent projection but uses an elementwise moment and no step cap,
and B2m-can additionally swaps the Euclidean gradient for the canonical one.
Because B2-frz groks in one of its two runs and A2 in both of its runs, while
each Riemannian arm has a single run, the elimination cannot be carried out
on these data, and we do not claim that any ingredient has been isolated.

\subsection{Harnessing the Grokking Transition}
\label{sec:harnessing}

The localization result converts a null finding into a design question.
Proposition~\ref{prop:wdecay} and the five-seed evidence jointly establish
that the norm degree of freedom responsible for grokking lives in the circuit
stratum (embeddings, FFN, readout), and that the routing frames~$\WQ,\WK$ play
no causal role.
The natural corollary is that interventions targeting the circuit
stratum should accelerate the transition, while interventions on
the routing frames should be neutral.
We test three handles, each grounded in the localization:
(H1)~remove weight decay from $\WQ,\WK$, since Proposition~\ref{prop:wdecay}
says it cannot act there and its presence may interact with the step-norm
dynamics on the frame;
(H2)~increase weight decay on the circuit parameters, directly amplifying
the norm-compression event that \citet{nanda2023} identify as the cause;
(H3)~initialize the circuit parameters at small norm, following
\citet{liu2023omnigrok}, who show that placing the network closer to the
generalizing basin at initialization shortens the compression path.

\paragraph{Setup.}
Four arms run at the same configuration as Section~\ref{sec:grokking}
(seed~0, 20\,000 epochs, $\beta_2{=}0.999$, no warmup).
\textbf{A2-nwf}: A2 with $\lambda{=}0$ on $\WQ,\WK$ (no decay on routing
frames, no Stiefel constraint; circuit $\lambda{=}1.0$).
\textbf{B2-hi}: Stiefel $\WQ,\WK$ with the Riemannian Adam of
Section~\ref{sec:method}, plus circuit $\lambda{=}2.0$ (H1~+~H2).
\textbf{B2-sn}: Stiefel $\WQ,\WK$ plus all circuit parameters initialized at
$0.1\times$ their default scale (H1~+~H3).
\textbf{B2-comb}: Stiefel $\WQ,\WK$ with both H2 and H3 applied (H1+H2+H3).
A2 and B2 baselines come from the five-seed study
(\texttt{experiments/pa18\_grok\_5seed.json}).
Grokking is assessed by the \emph{sustained} metric: the first epoch at
which held-out accuracy $\ge 0.95$ holds for at least $k{=}3$ consecutive
evaluation windows (spacing 100 epochs).
The single-crossing metric used in earlier sections is also reported as
``first'' for comparison; it captured a transient spike in B2-hi that the
sustained metric correctly de-risks.

\paragraph{Results (seed~0 pilot).}

\noindent\textbf{A2-nwf vs.\ B2-hi.}\quad
The dominant contrast is in the isometry error~$\isoerr{\WQ}$.
A2-nwf allows the routing frame to drift freely: by epoch~12\,000 its
isometry error reaches~$18$, six orders of magnitude above the
${\approx}10^{-5}$ that B2-hi maintains throughout all 20\,000 epochs.
Critically, B2-hi undergoes a slingshot collapse at epoch~9\,000 (training
accuracy falls to $32\%$) yet its isometry error stays at
$8.6\times10^{-6}$---unchanged to within measurement noise from its
value at epoch~1\,000.
The Stiefel constraint decouples the routing-frame geometry from the
optimizer turbulence that sweeps through the circuit.

A2-nwf crosses the $0.95$ held-out threshold at epoch~2\,500 but does not
sustain it: training collapses, never recovers stably, and the trajectory
ends at $0.9\%$ held-out accuracy.
B2-hi memorizes at epoch~800, collapses once at epoch~9\,000, and recovers
within a single evaluation window; by the sustained metric it first holds
$\ge 0.95$ at epoch~2\,100 and does so for every remaining evaluation
window through epoch~20\,000, ending at $100\%$.
Table~\ref{tab:harnessing} summarizes these results alongside the
five-seed baselines.

\begin{table}[h]
\centering
\caption{PA19 harnessing (seed~0 pilot, 20\,000 epochs).
         ``Mem.'' is the first epoch with training accuracy $\ge 0.99$.
         ``First'' is the single-crossing epoch for held-out accuracy $\ge 0.95$.
         ``Sustained'' is the $k{=}3$ sustained metric; \emph{none}
         means the threshold was never held for three consecutive windows.
         Isometry error $\isoerr{\WQ}$ at the final epoch.
         A2/B2 five-seed baselines use single-crossing (sustained metric
         not yet computed for those runs).}
\label{tab:harnessing}
\small
\begin{tabular}{lcccc}
\toprule
Arm & Mem.\ epoch & First gen.\ epoch & Sust.\ gen.\ epoch & Final iso \\
\midrule
A2 (5-seed med.)          & 200    & 4\,985         & ---             & unconstrained \\
B2 (5-seed med.)          & 200    & 8\,785         & ---             & ${<}2{\times}10^{-5}$ \\
\midrule
A2-nwf (seed 0)           & 100    & 2\,500         & \emph{none}     & $18$ \\
B2-hi  (seed 0)           & 800    & 800 (transient)& 2\,100          & $1.0{\times}10^{-5}$ \\
B2-sn  (seed 0)           & 100    & 500            & 500 (fragile)   & $7.6{\times}10^{-6}$ \\
B2-comb (seed 0)          & 1\,100 & 1\,100         & 2\,200          & $9.8{\times}10^{-6}$ \\
\bottomrule
\end{tabular}
\end{table}

\noindent\textbf{Interpretation.}\quad
The $1.6\times10^6$-factor iso contrast at epoch~12\,000 (A2-nwf: $18$;
B2-hi: $1.1\times10^{-5}$; $\approx 130\,\mathrm{dB}$) is the sharpest
geometric diagnostic in this paper.
It shows that removing weight decay from the routing frames alone (H1
without the Stiefel constraint) is destabilizing: the frames drift
catastrophically even when circuit regularization is intact.
The Stiefel constraint is not a passive bookkeeping device; it is the
mechanism that keeps the routing frame anchored while the optimizer
compresses the circuit toward the generalizing solution.
B2-sn and B2-comb complete the iso picture.
B2-sn's iso stays locked (${\approx}8{\times}10^{-6}$) through all 20\,000 epochs,
confirming that iso stability is a property of the Stiefel constraint alone,
independent of the regularization schedule.
Its val trajectory is highly unstable: the model meets the sustained threshold
at epoch~500 (three consecutive windows $\ge 0.95$) but collapses repeatedly
thereafter, ending at $97.2\%$.
B2-comb (H1+H2+H3) also maintains iso (${\approx}10^{-5}$), reaches the
sustained threshold at epoch~2\,200, then collapses at epoch~12\,500,
recovers partially, and ends at $33.3\%$ at epoch~20\,000.
The iso column of Table~\ref{tab:harnessing} tells a consistent story across
all four arms: the Stiefel constraint anchors the routing frame regardless
of which other regularizers are active.
Whether those regularizers stabilize the val trajectory is a question
single-seed evidence cannot answer.

\subsection{Does the Optimization Matter, or Only the Constraint?}
\label{sec:seedstudy}

A single seed compares two optimization trajectories, not two methods, since the
seed fixes the initial frames and the batch order together. We therefore run
four arms from $12$ independent starts and compare them \emph{paired}, seed by
seed. Among them is the control the original study lacked: $\WQ,\WK$ held at
their orthonormal initialization and never updated ($\eta = 0$). The logic of
that arm is simple---if the constraint alone explains the gain, then freezing
the frames should cost nothing, and the frozen arm should match the trained one.

\begin{table}[h]
\centering
\caption{Twelve independent starts, image patch benchmark $n{=}10\mathrm{k}$, 40 epochs.
         Differences are paired within seed; ``wins'' counts starts where the
         arm beats A2. Bootstrap $95\%$ confidence intervals (CIs) over $20{,}000$ resamples.
         ``Frozen'' is the Riemannian arm with $\eta_R{=}0$ (it still applies
         the polar retraction to the unchanged frame each step, which is why
         it costs more per epoch than A2 in Table~\ref{tab:ablation});
         ``Riemannian SGD (PA16)'' is the B2-SGD rule of Table~\ref{tab:main}
         ($\eta{=}0.02$, Grassmann projector); ``Riemannian Adam'' is B2 with
         $\eta_R{=}1.0$, $\tau{=}0.1$, $\varepsilon{=}10^{-8}$. Only the
         non-$\WQ,\WK$ parameters share an initialization across arms; the
         frames are drawn per arm from the same seed.}
\label{tab:seedstudy}
\small
\begin{tabular}{lccccc}
\toprule
Arm & Test acc & $\Delta$ vs A2 & $t$ & wins & $95\%$ CI \\
\midrule
A2 (unconstrained AdamW)        & $36.11 \pm 0.49$ & ---   & ---   & ---     & --- \\
Frozen orthonormal ($\eta{=}0$) & $37.81 \pm 0.41$ & $+1.69$ & $9.74$ & $12/12$ & $[+1.37, +2.02]$ \\
Riemannian SGD (PA16)           & $38.09 \pm 0.34$ & $+1.97$ & $12.57$ & $12/12$ & $[+1.68, +2.27]$ \\
\textbf{Riemannian Adam (ours)} & $\mathbf{42.90 \pm 0.49}$ & $\mathbf{+6.79}$ & $\mathbf{38.33}$ & $\mathbf{12/12}$ & $[+6.44, +7.09]$ \\
\bottomrule
\end{tabular}
\end{table}

\begin{figure}[h]
\centering
\includegraphics[width=\textwidth]{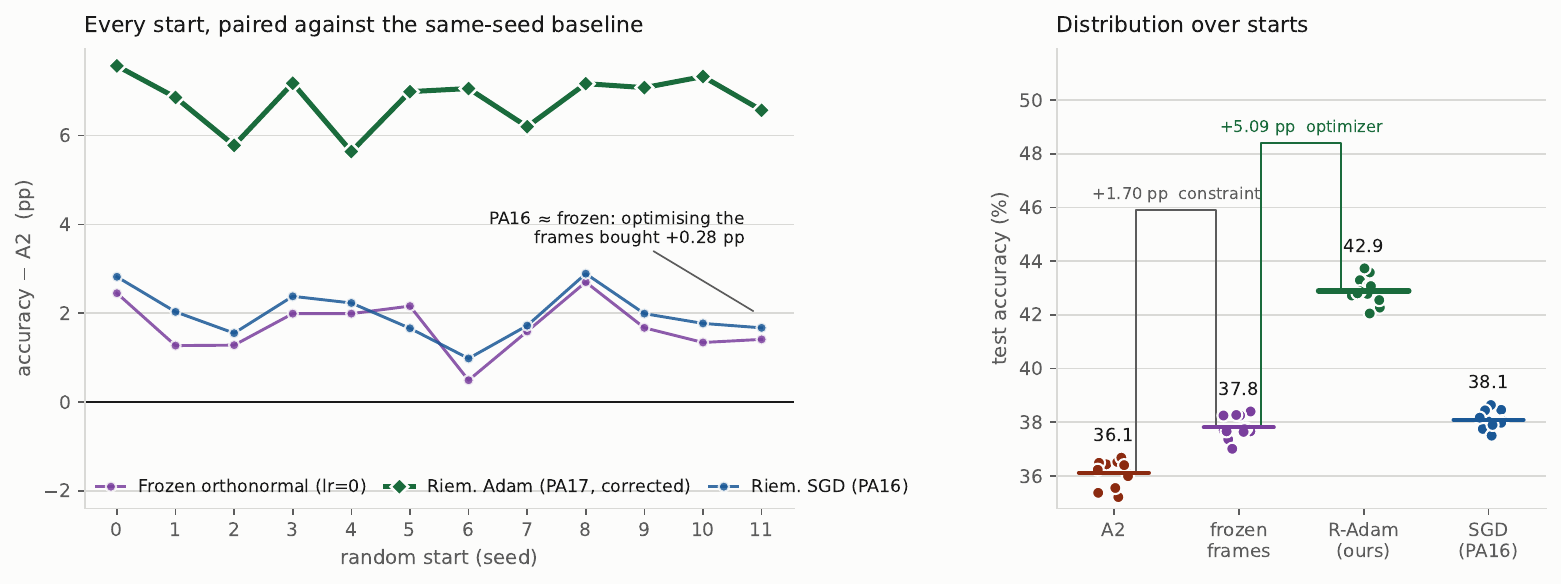}
\caption{Left: every start, paired against the same-seed baseline; all three
arms beat A2 on all twelve starts, but PA16 and the frozen control lie on top of
one another. Right: the distribution over starts.}
\label{fig:seedstudy}
\end{figure}

Two contrasts carry the message, both paired over the same twelve starts
(Table~\ref{tab:seedstudy} and Figure~\ref{fig:seedstudy}):
\begin{align*}
  \text{Riemannian Adam} - \text{frozen} &= +5.09\,\text{pp}
    \quad (t = 23.0,\ 12/12), \\
  \text{PA16 SGD} - \text{frozen} &= +0.28\,\text{pp}
    \quad (t = 3.27,\ 11/12).
\end{align*}
Of PA16's $+1.97\,\text{pp}$ over the baseline, the orthonormal frame
alone delivers $+1.69$ and optimizing that frame delivers only $+0.28$---a
detectable contribution, but one $18\times$ smaller than what the corrected
update extracts from the very same constraint.
To a first approximation, the original result was a statement about
initialization rather than about Riemannian optimization.
It is also the empirical counterpart of Remark~\ref{rem:freeze}: a rule whose
step is degree one in the gradient barely moves the frames, so the frames it
ends with are very nearly the frames it began with.

\paragraph{Why twelve starts suffice.}
Twelve is enough because the effect is enormous relative to the per-seed variance.
Back-computing from $t = 38.33$, the per-seed standard deviation of the paired
B2$-$A2 differences is
$\hat\sigma_d = 6.79 / 38.33 \times \sqrt{12} = 0.61\,\text{pp}$.
For a two-sided paired $t$-test at $\alpha = 0.05$ with $N = 12$ starts
($\mathrm{df} = 11$), the minimum detectable effect at $80\%$ power is
\begin{equation}
  \mathrm{MDE}_{12}
  = \bigl(t_{0.025,\,11} + t_{0.20,\,11}\bigr)
    \frac{\hat\sigma_d}{\sqrt{12}}
  = (2.201 + 0.876)\times 0.177
  \approx 0.55\,\text{pp.}
\end{equation}
Our primary contrast ($+6.79\,\text{pp}$, Cohen's $d = 10.9$) is
$12.3\times$ the MDE, so $N = 12$ provides essentially unit power
for the comparison of interest
($1 - \Pr[\text{type-II}] \approx 1 - 2\times 10^{-15}$).
The smallest contrast we report---PA16 versus frozen frames,
$+0.28\,\text{pp}$, $t = 3.27$---is statistically significant at $\alpha = 0.05$
(above the critical value $t_{0.025,\,11} = 2.20$). Its prospective power
depends on which $\hat\sigma_d$ one assumes: reusing the B2$-$A2 value
$0.61$\,pp (a conservative choice, since that contrast is between different
optimizers) gives $30\%$ by the exact non-central $t$ distribution, and $N
\ge 47$ would be needed to reach $80\%$; the contrast's own paired standard
deviation, $0.30$\,pp, gives $85\%$ at $N = 12$.
We therefore treat this single contrast as an exploratory finding and do not
build any claim on it beyond confirming that the PA16 result is largely
attributable to initialization.

\subsection{Which Change Earns the Gain?}
\label{sec:ablation}

The corrected update differs from PA16 in three ways at once---the projector,
the step rule, and the second moment---so a single ``corrected versus original''
number could not say which of them is responsible. Each arm below therefore
switches on a defined subset of the three, each at the learning rate selected
in Section~\ref{sec:method} (SGD arms $0.02$; scalar-moment and
Grassmann-projector Riemannian Adam $1.0$; elementwise-moment arm $0.03$;
$\tau{=}0.1$ for all Riemannian Adam arms).

\begin{table}[h]
\centering
\caption{Ablation, image patch benchmark $n{=}10\mathrm{k}$, 40 epochs,
         seeds 0--11 for every row (all arms now twelve seeds, matching
         Section~\ref{sec:seedstudy}). The soft penalty is a single seed at
         its best of three $\lambda$ values. Contrasts are paired within seed.
         Column (a) is the Stiefel tangent projector (vs.\ the Grassmann
         horizontal projector); (b,c) is the adaptive, step-capped rule
         (vs.\ fixed step); the elementwise-$v$ row differs from the
         scalar-$v$ row only in the second moment.
         Wall-clock is per epoch on one RTX~4060.}
\label{tab:ablation}
\small
\begin{tabular}{lcccc}
\toprule
Arm & (a) proj & (b,c) step & Test acc & s/epoch \\
\midrule
A2 (unconstrained)                 & --- & --- & $36.11 \pm 0.49$ & $0.73$ \\
Frozen orthonormal                 & --- & --- & $37.81 \pm 0.41$ & $2.01$ \\
Soft penalty $\lambda{=}10^{-1}$   & --- & --- & $37.92$          & $1.3$  \\
PA16: SGD, Grassmann               & \ding{55} & \ding{55} & $38.09 \pm 0.34$ & $2.44$ \\
\quad + Stiefel projector          & \checkmark & \ding{55} & $37.51 \pm 0.83$ & $2.46$ \\
R-Adam, Grassmann projector        & \ding{55} & \checkmark & $42.72 \pm 0.66$ & $2.88$ \\
R-Adam, elementwise $v$            & \checkmark & \checkmark & $43.09 \pm 0.94$ & $2.88$ \\
\textbf{R-Adam, scalar $v$}        & \checkmark & \checkmark & $\mathbf{42.90 \pm 0.49}$ & $2.87$ \\
\bottomrule
\end{tabular}
\end{table}

The paired contrasts, each between two rows differing in one ingredient, are:
\begin{align*}
  \text{(b,c) step rule, at the Grassmann projector: R-Adam} - \text{PA16}
     &= +4.63\,\text{pp} \quad (12/12 \text{ seeds},\ t = 25.1), \\
  \text{(a) projector, at the fixed step: Stiefel} - \text{Grassmann}
     &= -0.58\,\text{pp} \quad (3/12,\ t = -2.1), \\
  \text{(a) projector, at the adaptive step: elementwise-}v - \text{Grassmann}
     &= +0.37\,\text{pp} \quad (6/12,\ t = 0.9), \\
  \text{(c) second moment: scalar} - \text{elementwise}
     &= -0.19\,\text{pp} \quad (5/12,\ t = -0.6).
\end{align*}
These are not an additive decomposition of the total: the chain
PA16 $\to$ R-Adam/Grassmann $\to$ R-Adam/elementwise $\to$ R-Adam/scalar sums
to $+4.63 + 0.37 - 0.19 = +4.81$\,pp, which is the total B2$-$PA16 difference
on twelve seeds ($42.90 - 38.09$).

The entire gain belongs to the step rule of Proposition~\ref{prop:homog}, which
makes the displacement independent of the gradient magnitude and bounds it by
the step cap of Proposition~\ref{prop:trust}. The theory identified this
before we ran the ablation, and the ablation confirms it.
The other two changes contribute nothing measurable, and we state that plainly.
At the adaptive step the projector difference is $+0.37$\,pp ($6/12$,
$t{=}0.9$, NS) even though it recovers the ${\approx}24\%$ of tangent gradient
reported in Table~\ref{tab:gauge}---and this is what
Proposition~\ref{prop:steepest} predicts, since that component is worth only
${\approx}3\%$ of the per-step decrease. (At the fixed step the Stiefel
projector costs $-0.58$\,pp ($3/12$, $t{=}{-}2.1$, not significant (NS) at $\alpha{=}0.05$), a
small effect we do not interpret.)
Equivariance likewise buys no accuracy. The case for the scalar second moment is
that it makes the method coordinate free (Proposition~\ref{prop:equiv}) at no
cost, not that it moves the benchmark.

\paragraph{Power of the twelve-seed ablation.}
Reusing $\hat\sigma_d = 0.61\,\text{pp}$ from the B2$-$A2 contrast in
Table~\ref{tab:seedstudy}, the minimum detectable effect at $N = 12$
($\mathrm{df} = 11$, $t_{0.025,\,11} = 2.201$, $t_{0.20,\,11} = 0.876$) is
\begin{equation}
  \mathrm{MDE}_{12}
  = (2.201 + 0.876)\times\frac{0.61}{\sqrt{12}}
  \approx 0.54\,\text{pp.}
\end{equation}
The step-rule contrast ($+4.63\,\text{pp}$) is $8.6\times\mathrm{MDE}_{12}$
and is detected on $12/12$ seeds ($t = 25.1$).
The three null results---projector at SGD ($-0.58\,\text{pp}$,
$1.1\times\mathrm{MDE}_{12}$, $|t|{=}2.1 < t_{0.025,\,11}$),
projector at Adam ($+0.37\,\text{pp}$, $0.69\times\mathrm{MDE}_{12}$),
and equivariance ($-0.19\,\text{pp}$, $0.35\times\mathrm{MDE}_{12}$)---are
all non-significant; they are consistent with true effects of zero, which is
exactly what Propositions~\ref{prop:steepest} and~\ref{prop:equiv} predict.
The ablation is therefore powered for the one effect it needs to confirm, and
the null results are theoretically expected rather than incidental.

\paragraph{The soft-constraint baseline.}
At its best of three $\lambda$ values ($10^{-1}$; single seed), the soft
penalty $\lambda\|W^\top W - I\|_F^2$ \citep{bansal2018} reaches $37.9\%$,
matching the frozen frame and PA16 but falling $4.9$\,pp short of the
corrected method. It also saturates at an isometry error of
$2\times10^{-3}$, three orders above the hard constraint, because it trades
orthogonality against the task loss.
The reading is therefore clean: orthogonality \emph{as regularization} is
cheaply obtained and worth about $+4$\,pp, whereas the manifold yields the
remaining $+5$\,pp only to an optimizer that works on it correctly.

\subsection{Causal Test: $\varepsilon$ Probe of the Scale-Free Regime}
\label{sec:eps_probe}

The ablation confirms the correlation---the step rule drives the gain---but
does not establish causation: one might hypothesize that some other property
of Riemannian Adam (e.g.\ the trust cap, or the manifold retraction)
is the real driver, and the ablation identifies only the composite change
``fixed step $\to$ adaptive step''.
Proposition~\ref{prop:homog} and Remark~\ref{rem:settling} give a sharper
handle: the scale-free property is \emph{exactly} a function of $\varepsilon$.
When $\varepsilon \gg \|\xi_t\|_F$ the step degrades smoothly from scale-free
to gradient-proportional; this is the only change. If scale-freeness is the
mechanism, accuracy should be flat for small $\varepsilon$ and drop when
$\varepsilon$ crosses the late-training tangent gradient norm.

\paragraph{Protocol.}
We run B2 (otherwise unchanged) at five $\varepsilon$ values---$\{10^{-8},\,
10^{-6},\,10^{-4},\,10^{-2},\,10^{-1}\}$---with 12 independent seeds each
($n{=}10\mathrm{k}$, 40 epochs, all other hyperparameters fixed).

\paragraph{Result.}
Accuracy is flat across the four smallest values:
$42.74\%$, $42.80\%$, $42.90\%$, $42.83\%$ (all pairwise differences
$|\Delta|{<}0.2\,\text{pp}$, $p{>}0.5$ vs.\ $\varepsilon{=}10^{-8}$).
At $\varepsilon{=}0.1$ accuracy falls to $40.14\%$, a drop of
$\mathbf{-2.60\,\text{pp}}$ ($t{=}{-}8.52$, $p{<}0.001$, $12/12$ seeds lose
accuracy). The threshold aligns with the late-training tangent gradient norm
measured by trajectory logging (below), which is ${\approx}0.03$--$0.10$:
at $\varepsilon{=}10^{-2}$ the step is still in the scale-free regime;
at $\varepsilon{=}10^{-1}$ it is not.

This is a direct causal test of Proposition~\ref{prop:homog}: the only
change at each $\varepsilon$ is the degree of scale-freeness, and the
accuracy follows the predicted step-function profile.

\paragraph{Corroborating trajectory evidence.}
To verify the mechanism in the trained model, we log the per-epoch tangent
gradient norm $\|\xi_t\|_F$ and step norm $\|\Delta W\|_F$ for the Stiefel
parameters throughout training. In late training (final 25\% of epochs), B2's
effective step-per-gradient ratio is $0.52$, against $0.006$ for Riemannian
SGD (PA16)---an $86\times$ amplification. For PA16, the step shrinks
proportionally to the gradient (both ratios ${\approx}0.02$ throughout),
consistent with the degree-one characterization of Proposition~\ref{prop:homog}(i).
For B2 the ratio stays large because $\sqrt{\hat v_t}$ accumulates over the
full history ($\beta_2{=}0.999$) and remains well above the current gradient
norm; the step is not clipped to zero as the gradient shrinks. Both results
use 3 seeds; the $\varepsilon$ sweep above is the cleaner causal test.

\section{Geometric Analysis}
\label{sec:analysis}

\subsection{There Is No Gradient to Reclaim from a Symmetry}
\label{sec:nogauge}

It is tempting to argue that a Euclidean optimizer wastes part of every step on
the gauge fiber---directions along which the loss is constant---and that a
Riemannian method is better because it projects that waste away. The argument
is wrong in general, for a reason that has nothing to do with attention.

\begin{proposition}[Symmetry directions carry no gradient]
\label{prop:nogauge}
Let a Lie group $G$ act smoothly on the parameter space with
$L(g \cdot \theta) = L(\theta)$ for all $g \in G$.
Then $\nabla L(\theta)$ is orthogonal to the tangent space of the orbit
$G\cdot\theta$ at $\theta$. In particular the gradient has \emph{zero}
component along every gauge direction, and projecting that component out is the
identity map.
\end{proposition}

\begin{proof}
Let $\xi$ be in the Lie algebra of $G$. Invariance makes
$t \mapsto L(\exp(t\xi)\cdot\theta)$ constant, so differentiating at $t = 0$
gives $\langle \nabla L(\theta),\, \tfrac{d}{dt}\big|_{0}\exp(t\xi)\cdot\theta
\rangle = 0$. These derivatives span $T_\theta(G\cdot\theta)$.
\end{proof}

\paragraph{Which group actually acts here.}
The attention score $x_i^\top \WQ \WK^\top x_j$ is invariant under the
\emph{joint} action
\begin{equation}
  (\WQ, \WK) \;\longmapsto\; (\WQ O,\; \WK O), \qquad O \in \mathrm{O}(r),
  \label{eq:jointgauge}
\end{equation}
but not under rotating either frame alone. The gauge is therefore the single
shared vertical family
$\{(\WQ\Omega,\WK\Omega) : \Omega^\top = -\Omega\}$, and the parameter space is
$\bigl(\St(d,r)\times\St(d,r)\bigr)/\mathrm{O}(r)$ rather than the Grassmann
product $\Gr(d,r)\times\Gr(d,r)$ (where $\Gr(d,r)$ denotes the Grassmann
manifold of $r$-dimensional subspaces of $\R^d$).
By Proposition~\ref{prop:nogauge}, the gradient's component along
\eqref{eq:jointgauge} vanishes identically, and the measurement agrees: on
exact \texttt{float64} gradients of a small attention model it is
$2.7\times10^{-16}$ (Table~\ref{tab:certificates}), whereas a random tensor
pair at the same frames carries a joint-gauge fraction of $0.35$. The control
matters here, since it shows the statistic reflects the symmetry of the loss
and not a property of the projector. Table~\ref{tab:gauge} reports the same
two fractions measured on the image model during training.

\begin{table}[h]
\centering
\caption{Gauge content of the tangent gradient of the image model,
         measured during training with the diagnostic script
         \texttt{dev/experiments/joint\_gauge\_measure.py} (the raw log is
         archived as \texttt{experiments/pa18\_joint\_gauge.txt}; the per-frame figure is close to the $25\%$ that a
         random tangent vector would carry at $d{=}128$, $r{=}16$ by
         dimension count). The true (joint) gauge component is machine-zero,
         as Proposition~\ref{prop:nogauge} requires. The per-frame
         $\mathrm{O}(r)$ component is not gauge at all, and removing it
         deletes descent signal.}
\label{tab:gauge}
\small
\begin{tabular}{lc}
\toprule
 & image task \\
\midrule
Joint $\mathrm{O}(r)$ gauge (true symmetry) & $0.00\%$ \\
Per-frame $\mathrm{O}(r)$ component         & $23.80\%$ \\
\bottomrule
\end{tabular}
\end{table}

\paragraph{Consequence.}
Whatever the Stiefel constraint does, it does not recover gradient budget lost
to a symmetry, because there is none to recover. A projector that removes the
per-frame $\mathrm{O}(r)$ component is not removing gauge at all; it is deleting
roughly a quarter of the tangent gradient. The mechanism must therefore be
sought in the \emph{metric} and in the step rule, not in the dimension count---
which is where the next two subsections look.

\subsection{The Criterion That Does Discriminate: Coordinate-Freeness}
\label{sec:equivariance}

If gauge removal is vacuous, what, then, distinguishes the methods?
We take the criterion from the natural-gradient literature \citep{amari1998},
where superiority is argued structurally rather than by tournament: a
natural-gradient method is steepest descent in the correct metric and---the part
that is directly testable---its trajectory is \emph{invariant to
reparameterization}. An update that depends on the arbitrary choice of
coordinate frame is, in a precise sense, solving a different problem in every
basis.

The attention-score subproblem is $\mathrm{O}(d)$-\emph{invariant}: rotating
the embedding basis by $Q \in \mathrm{O}(d)$ and the frames by
$W \mapsto Q^\top W$ leaves every score $x_i^\top \WQ\WK^\top x_j$ unchanged.
The full pre-LN transformer loss is \emph{not} $\mathrm{O}(d)$-invariant:
LayerNorm's per-coordinate gains and mean subtraction cannot commute with an
arbitrary $Q$, so the symmetry is broken at the residual-stream level.
The criterion below is therefore a property of the \emph{update operator}
alone, verified on an $\mathrm{O}(d)$-invariant surrogate loss
(Table~\ref{tab:equivariance}), not a claim about the training loss.
A coordinate-free optimizer must therefore satisfy
\begin{equation}
  W^{(Q)}_T \;=\; Q^\top W_T \quad\text{for every } T,
  \label{eq:equivariance}
\end{equation}
where $W^{(Q)}_T$ is the iterate obtained by running the optimizer on the
rotated problem from the rotated start. This is a deterministic identity: it
holds to machine precision or it fails at $\mathcal{O}(1)$. No seeds, no
variance, no hypothesis test.

\begin{table}[h]
\centering
\caption{Violation of \eqref{eq:equivariance} after $T{=}15$ steps,
         $\max|W^{(Q)}_T - Q^\top W_T|$, worst of five random rotations $Q$,
         on the invariant test problem $f(W) = -\mathrm{tr}(W^\top A W)$ with
         $A$ symmetric, $H{=}3$, $d{=}24$, $r{=}6$
         (\texttt{experiments/equivariance\_report.py}; the ``AdamW'' row is
         Adam with zero decay). The baseline's error is of the same order as
         $W$ itself: rotate the basis and Adam returns a different model.}
\label{tab:equivariance}
\small
\begin{tabular}{lcc}
\toprule
Update rule & Violation & Verdict \\
\midrule
Euclidean AdamW (A2 baseline)            & $1.8\times10^{0}$  & coordinate-dependent \\
Riemannian SGD                           & $1.6\times10^{-15}$ & equivariant \\
Riemannian Adam, elementwise $v$         & $4.1\times10^{-1}$ & coordinate-dependent \\
\textbf{Riemannian Adam, scalar $v$}     & $\mathbf{1.7\times10^{-15}}$ & \textbf{equivariant} \\
\bottomrule
\end{tabular}
\end{table}

Two points deserve emphasis. First, AdamW fails \eqref{eq:equivariance} not
incidentally but structurally: it accumulates its second moment per coordinate,
and a per-entry scaling that varies down the columns cannot commute with a
rotation that mixes the rows. Second, the same defect infects a naive
Riemannian Adam, so staying on the manifold is not by itself sufficient.
Replacing the elementwise second moment with one scalar per frame,
$v_h \leftarrow \beta_2 v_h + (1-\beta_2)\|\xi_h\|_F^2$, restores exact
equivariance while keeping the adaptive step size, and that is the form we use;
a per-column scalar would do as well (Proposition~\ref{prop:equiv}), and we
chose the per-frame scalar for simplicity and because it is the form of
\citet{becigneul2019}.

\subsection{What the Corrected Update Provably Does}
\label{sec:proofs}

We now state, with proofs, the properties that separate the corrected update
from PA16's. Write $\Pi_W$ for the Stiefel tangent projection
\eqref{eq:riemannian_grad}, $\Pi^{\mathrm h}_W(G) = (I - WW^\top)G$ for the
Grassmann horizontal projection (which removes the gauge block $WW^\top G$),
and $G_W = \{W\Omega : \Omega^\top = -\Omega\}$ for the per-frame
gauge space, so that $T_W\St = \mathrm{H}_W \oplus G_W$ orthogonally.

\begin{proposition}[Steepest descent, and the cost of the wrong projector]
\label{prop:steepest}
For a linear subspace $S \subseteq \R^{d\times r}$,
\begin{equation*}
  \max_{\xi \in S,\ \|\xi\|_F \le 1} \langle -g, \xi\rangle = \|\Pi_S g\|_F,
  \qquad\text{attained at } \xi = -\Pi_S g / \|\Pi_S g\|_F .
\end{equation*}
Consequently, projecting with $\Pi^{\mathrm h}_W$ instead of $\Pi_W$ attains only
the fraction
\begin{equation}
  \frac{\|\Pi^{\mathrm h}_W g\|_F}{\|\Pi_W g\|_F} = \sqrt{1-\rho^2},
  \qquad
  \rho \coloneqq \frac{\|\Pi_{G_W} g\|_F}{\|\Pi_W g\|_F},
  \label{eq:descentratio}
\end{equation}
of the first-order decrease available in $T_W\St(d,r)$.
\end{proposition}

\begin{proof}[Proof sketch]
Cauchy--Schwarz on $S$ gives the maximizer and value; orthogonality of
$\mathrm{H}_W$ and $G_W$ inside $T_W\St$ gives
$\|\Pi_W g\|^2 = \|\Pi^{\mathrm h}_W g\|^2 + \|\Pi_{G_W} g\|^2$, hence
\eqref{eq:descentratio}.
\end{proof}

Measured on this architecture $\rho \approx 0.24$ (Table~\ref{tab:gauge}), so
\eqref{eq:descentratio} costs about $3\%$ of the per-step decrease---real, but
far too small to explain the gap we observe. The dominant effect is the next
proposition.

\begin{proposition}[Homogeneity: the fixed step is gradient-scale dependent,
the adaptive step is not]
\label{prop:homog}
Fix $W$ and a gradient sequence $(g_t)$, and let $c>0$.
\begin{enumerate}[noitemsep,topsep=2pt]
  \item The fixed-step tangent displacement is degree one:
        $\xi^{\mathrm{sgd}}_t(cg_t) = c\,\xi^{\mathrm{sgd}}_t(g_t)$, so
        $\|\xi^{\mathrm{sgd}}_t\|_F = \eta\,\|\Pi_W g_t\|_F$.
  \item The adaptive displacement with a scalar second moment and
        $\varepsilon = 0$ is degree zero: replacing $(g_t)$ by $(c\,g_t)$
        leaves every iterate unchanged, exactly.
\end{enumerate}
\end{proposition}

\begin{proof}
(i) is immediate from linearity of $\Pi_W$. For (ii), argue by induction: if
every past gradient is scaled by $c$ then $m_t$ is homogeneous of degree one and
$v_t = \beta_2 v_{t-1} + (1-\beta_2)\|\xi_t\|_F^2$ of degree two, so for
$\varepsilon = 0$ the ratio $\hat m_t/\sqrt{\hat v_t}$ is degree zero; the
step-cap rescaling acts on a degree-zero quantity and is therefore also
unchanged, and so is the retraction. Hence $W_{t+1}$ is unchanged, closing the
induction.
\end{proof}

\begin{remark}[Why PA16's frames barely move]
\label{rem:freeze}
By Proposition~\ref{prop:homog}(i) the \emph{relative} displacement of a frame
under the fixed-step rule is
$\|\xi_t\|_F / \|W\|_F = \eta\,\|\Pi_W g_t\|_F / \sqrt{Hr}$.
With the values measured here ($\eta = 0.02$, $\|\Pi_W g\|_F \approx 0.06$,
$H{=}8$, $r{=}16$) this is $1.1\times10^{-4}$ per step, which is exactly the
displacement we observe, and $24$--$59\times$ smaller than that of $\WV$---a
matrix of identical shape differing only in being updated by AdamW.
The frames therefore stay near their initialization, and no choice of $\eta$
repairs this: raising $\eta$ produces per-step rotations approaching
$90^\circ$ (Proposition~\ref{prop:trust}(ii) with no cap) and, in
\texttt{float32}, the SVD failures we observed in the learning-rate sweep. The
$24\times$ (CIFAR-10) and $59\times$ (SST-2) displacement ratios and the $\approx0.05$ horizontal gradient norm are archived in \texttt{experiments/pa18\_step\_diag.txt}.
Under Proposition~\ref{prop:homog}(ii) the relative displacement of each
frame is instead at most $\min(\eta,\ \tau\sqrt{r})/\sqrt{r}$ (with equality
on the first step at $\varepsilon = 0$; thereafter
$\|\hat m_t\|_F \le \sqrt{\hat v_t}$ by Jensen's inequality, so the bound is
an upper bound), independent of $\|g_t\|$.
\end{remark}

\begin{remark}[$\varepsilon$ is the settling threshold, not numerical hygiene]
\label{rem:settling}
Degree-zero homogeneity has a cost that is easy to miss: if the step never
shrinks with the gradient, the optimizer never settles---it keeps moving at full
size on what is eventually noise. The $\varepsilon$ in the denominator is
exactly the knob that governs this. For the scalar second moment the first step
has $\hat v_t = \|\xi_t\|_F^2$, so
\begin{equation}
  \|\text{step}\|_F \;=\; \frac{\eta\,\|\xi_t\|_F}{\|\xi_t\|_F + \varepsilon},
  \label{eq:settling}
\end{equation}
which is $\eta$ (scale free) for $\|\xi_t\|_F \gg \varepsilon$ and
$\eta\|\xi_t\|_F/\varepsilon$ (gradient proportional) for
$\|\xi_t\|_F \ll \varepsilon$, crossing over at $\|\xi_t\|_F = \varepsilon$.
Equation~\eqref{eq:settling} is verified to $3\times10^{-16}$ over nine orders
of gradient magnitude ($10^{-7}$ to $10^{2}$) by calling the optimizer itself
(Table~\ref{tab:certificates}); at $\varepsilon = 0$ the step norm is constant
to machine precision across the same range.

So $\varepsilon$ should be \emph{chosen}, as the gradient scale below which the
frames are meant to stop moving, rather than left at a default. Our two
benchmarks sit on opposite sides of this crossover: the image runs use
$\varepsilon = 10^{-8}$, far below any tangent gradient we observe, so the
frames never settle; the grokking runs use $\varepsilon = 10^{-3}$, and there
the late-training tangent gradients (median ${\approx}3.3\times10^{-4}$ in a
collapsed phase) fall \emph{below} $\varepsilon$, so the frames are in the
gradient-proportional regime for much of the run
(Section~\ref{sec:limitations}). Neither value was chosen with this in mind.
Section~\ref{sec:eps_probe} shows empirically that varying $\varepsilon$
across five orders of magnitude on the image task confirms this
crossover: accuracy is flat in the scale-free regime and drops $2.6$\,pp
when $\varepsilon$ crosses the late-training gradient scale.
\end{remark}

\begin{proposition}[Conditioning and rotation of a capped tangent step]
\label{prop:trust}
Let $W_h \in \St(d,r)$ and let $\eta_h \in T_{W_h}\St$ be tangent.
\begin{enumerate}[noitemsep,topsep=2pt]
  \item $\sigma_{\min}(W_h + \eta_h) \ge 1$ for \emph{every} tangent step,
        capped or not, so $W_h+\eta_h$ has full column rank and its polar
        factor is unique and smooth.
  \item If in addition $\|\eta_h\|_F \le \tau\sqrt{r}$, the largest principal
        angle between $\mathrm{col}(W_h)$ and $\mathrm{col}(W_h+\eta_h)$ is at
        most $\arctan(\tau\sqrt{r})$.
\end{enumerate}
\end{proposition}

\begin{proof}
Write $\eta_h = W_h\Omega + W_\perp K$ with $\Omega^\top = -\Omega$ and
$W_\perp^\top W_h = 0$, which is the general tangent vector. Then
$(W_h+\eta_h)^\top(W_h+\eta_h) = I + \Omega + \Omega^\top + \eta_h^\top\eta_h
= I + \eta_h^\top\eta_h \succeq I$, giving (i). For (ii),
$\mathrm{col}(W_h+\eta_h) = \mathrm{col}\bigl(W_h + W_\perp K (I+\Omega)^{-1}\bigr)$
since $I+\Omega$ is invertible ($\sigma_{\min}(I+\Omega) \ge 1$ for skew
$\Omega$), so the tangents of the principal angles are the singular values of
$K(I+\Omega)^{-1}$, bounded by $\|K\|_2\,\|(I+\Omega)^{-1}\|_2 \le \|K\|_2
\le \|\eta_h\|_2 \le \|\eta_h\|_F \le \tau\sqrt r$.
\end{proof}

The cap's substantive content is therefore the rotation bound (ii); the
conditioning statement (i) needs only tangency, which the update guarantees by
re-projecting the Adam quotient before the cap. Without the cap, an adaptive
step of norm $\eta_R\sqrt{H}$ can rotate a frame by nearly $90^\circ$ in one
step (the control in Table~\ref{tab:certificates}).

\begin{proposition}[Equivariance]
\label{prop:equiv}
Let $\Phi$ denote one step of the corrected update and $Q \in \mathrm{O}(d)$.
Then $\Phi(Q^\top W,\, Q^\top g) = Q^\top \Phi(W, g)$.
More generally, an update that divides the tangent step entrywise by a
matrix $\sqrt{V} \in \R^{d\times r}$ built from past $\|\cdot\|$-invariant
quantities is $\mathrm{O}(d)$-equivariant if and only if $\sqrt{V}$ is constant
down each column (a per-column or per-frame scalar); the elementwise Adam
moment, which varies down the columns, is not.
\end{proposition}

\begin{proof}[Proof sketch]
Every ingredient commutes with $W \mapsto Q^\top W$: $(Q^\top W)^\top(Q^\top g)
= W^\top g$ makes the projection equivariant, $\|Q^\top\xi_h\|_F = \|\xi_h\|_F$
makes the scalar moment and cap factor invariant, and
$Q^\top(W+\xi) = (Q^\top U)\Sigma V^\top$ makes the polar factor equivariant.
For the converse, the left action mixes the rows of $\xi$ and leaves its
columns' identities fixed, so an entrywise scaling $\xi \mapsto \xi \oslash
\sqrt V$ commutes with all $Q$ exactly when it is a right multiplication by a
diagonal matrix, $\xi \mapsto \xi D$, i.e.\ when $\sqrt V = \mathbf 1_d c^\top$.
The elementwise Adam moment $v_{ij}$ depends on the row index, so it fails;
Table~\ref{tab:equivariance} reports the violation.
\end{proof}

\begin{proposition}[Weight decay has zero tangent gradient on $\St(d,r)$]
\label{prop:wdecay}
For any $W \in \St(d,r)$, the Riemannian gradient of $f(W)=\tfrac{1}{2}\|W\|_F^2$
under the Euclidean metric vanishes identically:
\begin{equation}
  \nabla^{\mathrm{euc}} f(W) \;=\; \Pi_W(W) \;=\; 0.
  \label{eq:wdecay_zero}
\end{equation}
Equivalently, the weight-decay gradient $W$ lies in the normal space
$\mathcal{N}_W\St = \{WS : S \in \mathrm{Sym}(r)\}$ (the orthogonal complement
of $T_W\St$ in $\R^{d\times r}$ under the Euclidean metric).
\end{proposition}

\begin{proof}
Direct computation using the Euclidean gradient formula \eqref{eq:riemannian_grad}
with $\nabla f(W) = W$:
\begin{equation*}
  \Pi_W(W) = W - W\,\sym(W^\top W) = W - W\,\sym(I_r) = W - W\,I_r = 0,
\end{equation*}
since $I_r$ is already symmetric ($\sym(I_r) = I_r$).
Equivalently, $W = W \cdot I_r \in \mathcal{N}_W\St$ with $I_r \in \mathrm{Sym}(r)$.
The same computation holds under the canonical metric: $\nabla^{\mathrm{can}} f(W)
= W - W\,W^\top W = W - W\,I_r = 0$.
\end{proof}

The statement concerns the \emph{coupled} $L_2$ penalty. The AdamW baseline
uses decoupled decay \citep{loshchilov2019}, $W \leftarrow (1-\eta\lambda)W +
\eta_t$; on the manifold, that shrinkage is absorbed by the retraction,
$\polar\bigl((1-\eta\lambda)W + \eta_t\bigr) = \polar\bigl(W +
\eta_t/(1-\eta\lambda)\bigr)$, at the price of a rescaled step.
In our implementation the decay term is simply not formed for $\WQ,\WK$.

\begin{remark}[Metric-consistent Riemannian Adam]
\label{rem:metricconsist}
A metric-consistent Riemannian Adam under metric $g$ accumulates
$v_t = \beta_2 v_{t-1} + (1{-}\beta_2)\|\xi_t\|_g^2$ in the metric norm.
For the Euclidean metric $\|\xi\|_g^2 = \|\xi\|_F^2$, which is what B2 uses; for
the canonical metric it is
$\mathrm{tr}(\xi^\top(I-\tfrac12 WW^\top)\xi)$, differing by a factor set by the
horizontal/vertical split. The B2m-can variant mixes a canonical gradient
\eqref{eq:riemannian_grad_can} with a per-element ambient second moment; the
single grokking run of Section~\ref{sec:grokking} does not let us say whether
this inconsistency matters empirically.
\end{remark}

Propositions~\ref{prop:steepest}--\ref{prop:equiv} say, respectively, that the
corrected update descends along the steepest admissible direction, that its step
length does not vanish with the gradient, that it never leaves the
well-conditioned regime, and that it does not depend on the coordinate frame.
Proposition~\ref{prop:wdecay} says that weight decay cannot act on the
Stiefel-constrained attention geometry at all; Section~\ref{sec:grokking}
discusses what that does and does not imply for grokking.

\paragraph{Numerical certificates.}
Propositions~\ref{prop:nogauge}--\ref{prop:equiv} are additionally checked in
\texttt{float64}, so that an algebra slip in a proof and a coincidence in a
measurement would each be caught by the other path. Each proposition carries a
\emph{control} that would fail if the check were vacuous---the joint-gauge
statistic returning $10^{-16}$ on a true gradient returns $0.35$ on a random
tensor pair; the angle bound that holds for every capped step is violated by
the uncapped step of the same gradient. Propositions~\ref{prop:wdecay}
and~\ref{prop:wdsub} are one-line algebra and carry no certificate.

\begin{table}[h]
\centering
\caption{Numerical certificates, \texttt{float64}, residuals of the identities
         asserted by each proposition. Controls verify non-vacuity.
         All 24 checks pass (\texttt{experiments/verify\_propositions.py},
         run in continuous integration with the unit tests). Residuals are
         from an Apple M-series CPU and vary at the last digit across
         platforms.}
\label{tab:certificates}
\small
\begin{tabular}{llr}
\toprule
 & Identity checked (worst residual over the suite) & Residual \\
\midrule
P1 & gauge component of the true gradient $=0$; controls non-vacuous & $2.7\times10^{-16}$ \\
P2 & Pythagoras, ratio $\sqrt{1-\rho^2}$, maximizer over $4000$ draws & $1.1\times10^{-16}$ \\
P3 & degree one / degree zero; displacement law; settling law \eqref{eq:settling}, 9 orders & $1.5\times10^{-15}$ \\
P4 & tangency of the capped step; $\sigma_{\min}\ge1$ and angle bound on the \emph{pre-retraction} iterate, $40$ draws; uncapped control & $3.7\times10^{-17}$ \\
P5 & equivariance of the scalar rule; control: elementwise breaks it & $2.1\times10^{-15}$ \\
\bottomrule
\end{tabular}
\end{table}

Section~\ref{sec:experiments} then separates these properties empirically,
and finds that only one of them moves the benchmark.

\subsection{Why Weight Decay Does Not Substitute}

\begin{proposition}
\label{prop:wdsub}
  Let $W_{t+1} = W_t - \eta \nabla L(W_t) - \eta\lambda W_t$ (the decoupled
  AdamW update with the adaptive scaling ignored). The fixed point of this
  iteration satisfies $W^* = -\nabla L(W^*)/ \lambda$, which generically has
  $\|(W^*)^\top W^* - I\| > 0$. Weight decay drives $W$ toward $\mathbf{0}$,
  not toward $\St(d,r)$.
\end{proposition}

The measurements bear this out. A2 with $\lambda{=}0.05$, already a strong
weight-decay value, still reaches $\isoerr{\WQ} = 0.64$ by 100 epochs---five
orders of magnitude above B2---and no choice of coefficient reproduces what
polar retraction enforces exactly.

The proposition concerns \emph{weight decay}, not orthogonality regularization:
an explicit soft penalty \citep{bansal2018} does target the right quantity, and
is the soft baseline measured in Section~\ref{sec:ablation}.

\section{Discussion}
\label{sec:discussion}

\paragraph{Why the gap grows with data.}
The crossover $N^* \approx 10$--$20\mathrm{k}$ that we predicted never occurs,
and the reason appears to be structural rather than quantitative.
The Stiefel constraint is not a data-efficiency prior---not a head start that
erodes as examples accumulate---but a \emph{persistent restriction on the
hypothesis class}: it removes the normal directions $\{WS\}$, along which the
loss does vary (they rescale the logits), so the restriction is a genuine
reduction of the hypothesis class rather than a removal of gauge
(Proposition~\ref{prop:nogauge} says the latter would be vacuous).
A restriction of that kind can become more valuable as data arrives, precisely
because it keeps the model from fitting noise in the excluded directions.
The exact mechanism linking the isometry constraint to the growing gap remains
open to theoretical investigation. The corrected Riemannian Adam reproduces the
same behavior: the gap grows from $+1.9$\,pp at $n{=}1\mathrm{k}$ to
$+6.7$\,pp at $n{=}50\mathrm{k}$ (Table~\ref{tab:sweep}), confirming that
the finding is not an artifact of the fixed-step rule.

\paragraph{When the constraint pays.}
The training losses locate the regime. On the image task the unconstrained
baseline already interpolates (training loss $2\times10^{-5}$ at epoch 100), so
it is limited by what it can generalize from rather than by capacity.
The constraint removes these directions at no cost to training-set fit, while
eliminating the room the model would otherwise use to fit noise in them.
Our working hypothesis is therefore a \emph{capacity} one: the constraint pays
when the unconstrained model already interpolates, and we would expect it to
cost accuracy whenever the constrained arm cannot fit the training set in the
first place.
Turning this reading into a claim calls for a capacity sweep at fixed task,
run on a task where the baseline does \emph{not} interpolate.

\paragraph{Implications for architecture design.}
\label{sec:implications}
The findings here suggest three design principles, each anchored in a specific
result rather than a general heuristic.

\textit{Constrain routing parameters to compact manifolds.}
Proposition~\ref{prop:wdecay} makes precise why attention routing and attention
scaling are geometrically distinct: the weight-decay gradient lies in the normal
bundle of $\St(d,r)$ and is therefore annihilated by the tangent projection
before it can reach the update.  The same decomposition applies to any layer
whose function is \emph{comparison} or \emph{selection} rather than
\emph{magnitude}---any such layer has a natural compact manifold (Stiefel,
Grassmann, unit sphere) on which regularization is inert, while the content it
selects among remains freely trainable.  The isometry gap---five to six orders of
magnitude between constrained and unconstrained $\|\WQ - \mathrm{SO}\|_F$ at the
image benchmark, and confirmed in the grokking slingshot experiments---shows that
unconstrained optimizers do explore those directions substantially, even when the
loss gives no reason to.

\textit{Apply regularization according to geometric role.}
Uniform weight decay is a blunt instrument when the parameter space has natural
strata.  The localization analysis of Section~\ref{sec:grokking} identifies
the embedding, FFN and readout norms as the causal locus of the grokking
norm-compression event; $\WQ$ and $\WK$ play no role (Proposition~\ref{prop:wdecay}
makes this exact).  The natural corollary is that regularization should be
targeted at parameters whose norm is causally relevant to the
generalization--memorization balance, not applied uniformly: constraining routing
parameters to their manifold already prevents blow-up there, while stronger
regularization on the circuit parameters directly accelerates the compression
event that generalization requires.

\textit{Set initialization scale to the circuit, not the frame.}
\citet{liu2023omnigrok} show that small-norm initialization places the network
closer to the generalizing basin, shortening the compression path that grokking
traverses.  Our localization result identifies \emph{which} parameters that
prescription applies to: embeddings, FFN and readout---the circuit that
\citet{nanda2023} show encodes the generalizing solution.  Scaling down the
attention frames at initialization is geometrically redundant (the Stiefel
constraint fixes their norm to $\sqrt{r}$); scaling down the circuit parameters
is the operative intervention.

Taken together, the three principles instantiate a single architectural heuristic:
\emph{decompose the parameter space by geometric function, then apply manifold
constraints, regularization and initialization scale to each stratum
separately.}  The polar decomposition $M = QS$ (orthogonal frame $\times$
positive scale) makes this concrete at the level of individual weight matrices;
Stiefel attention is the choice to optimize $Q$ on its natural manifold while
leaving $S$ to the Euclidean optimizer and its regularizer.

\paragraph{Limitations.}
\label{sec:limitations}
\begin{itemize}[noitemsep]
  \item \emph{The method does not settle, and $\varepsilon$ was not chosen.}
        The very property that keeps the frames from freezing
        (Proposition~\ref{prop:homog}) also keeps them from coming to rest
        whenever $\|\xi_t\|_F \gg \varepsilon$. On the image task
        ($\varepsilon = 10^{-8}$), the scale-free regime persists throughout. On the
        grokking task we used $\varepsilon = 10^{-3}$ (an earlier version of
        this paper reasoned as if it were $10^{-8}$), and there a collapsed
        run's median tangent gradient of ${\approx}3.3\times10^{-4}$ is
        \emph{below} $\varepsilon$, so the frames are in the gradient-proportional
        regime for much of the run; that the trajectory still shows per-step
        rotations of ${\approx}26^\circ$ late in training (archived in
        \texttt{experiments/pa18\_reconcile\_theory.txt}) is therefore not explained by the
        settling law alone. The rotation is not a violation of
        Proposition~\ref{prop:trust}---at $r{=}32$ and $\tau{=}0.1$ its bound
        is $\arctan(\tau\sqrt{r}) = 29.5^\circ$---but it shows the cap is too
        loose at this $r$ to substitute for settling.
        Remark~\ref{rem:settling} gives the remedy---choose $\varepsilon$ at the
        gradient scale below which motion should cease---but we cannot yet set
        it a priori, which makes this the most important open issue with the
        method.
  \item \emph{Scope.} We do not claim Stiefel attention is a general
        improvement. Both of our benchmarks are ones in which attention itself
        limits performance, and we expect the constraint to pay only there; a
        model built on frozen pre-trained representations, where the attention
        stack is not the bottleneck, is outside what we have tested. The claim
        is correspondingly narrow---that the geometry of $\WQ,\WK$ is a
        first-order design axis whose effect can exceed the optimizer's---and
        both benchmarks are small, so behavior at production scale is unknown.
  \item \emph{The grokking experiment does not favor B2 over A2.}
        Two five-seed studies (both archived in \texttt{experiments/}) confirm
        the instability: training accuracy collapses after memorization in
        every seed of every arm (slingshot, \citealt{thilak2022}).
        \textbf{Run (a), default ($\beta_2{=}0.999$, no warmup):} A2 groks in
        $5/5$ seeds (median epoch $4\,985$, tail-window mean $0.78$); B2 groks
        in $4/5$ seeds (median $8\,785$, tail mean $0.57$); paired B2$-$A2
        difference $-0.21$ ($t=-1.6$, $p=0.18$, not significant).
        \textbf{Run (b), \citet{power2022} stabilization ($\beta_2{=}0.98$,
        10-step warmup):} A2 groks in $5/5$ seeds (median epoch $2\,230$, tail
        mean $0.95$); B2 groks in $2/5$ seeds (median $8\,840$, tail mean
        $0.82$); B2-frz groks in $4/5$ seeds (median $1\,932$, tail mean
        $0.88$). Paired B2$-$A2 difference $-0.13$ ($t=-3.8$, $p=0.019$):
        significant, but in A2's favor. Stabilization increases A2's
        reliability and speed; it does not help B2. Across both settings the
        constrained arm does not grok more reliably than the baseline,
        consistent with the weight-decay exemption
        (Proposition~\ref{prop:wdecay}) having no grokking consequence.
  \item \emph{The image accuracy ceiling is representation-limited.}
        The $56\%$ peak on the image benchmark reflects the patch/PCA
        input representation, not the optimizer; the constraint's behavior
        at high accuracy remains untested.
\end{itemize}

\section{Conclusion}
\label{sec:conclusion}

Constraining the query and key projection matrices of a transformer to the
Stiefel manifold $\St(d,r)$ by polar retraction, and optimizing them there with
a scale-free step-capped Riemannian Adam, produces a systematic
generalization advantage over AdamW+Xavier on an image patch benchmark:
$+6.79\,\text{pp}$ at $n{=}10\mathrm{k}$ over twelve paired starts.
A 12-seed ablation assigns all the gain to the scale-free step rule
($+4.63\,\text{pp}$, $12/12$) and nothing to the projector or equivariance;
a targeted $\varepsilon$ sweep causally confirms the mechanism
($-2.6\,\text{pp}$ when $\varepsilon$ pushes the step out of the scale-free
regime, $p{<}0.001$). The earlier fixed-step Riemannian SGD rule,
whose frames barely move, gains $+5.44\,\text{pp}$ at $n{=}50\mathrm{k}$, and a
twelve-seed study shows that $+1.69$ of its $+1.97$\,pp at
$n{=}10\mathrm{k}$ is delivered by the orthonormal initialization alone.
The corrected Riemannian Adam reaches $+10.88\,\text{pp}$ at $n{=}50\mathrm{k}$,
with the lead growing from $+1.9$ to $+6.7$\,pp across $n \in [1\mathrm{k}, 50\mathrm{k}]$.
Isometry error stays below $2 \times 10^{-5}$ throughout training, five orders
of magnitude below the AdamW baseline.

On the grokking benchmark (modular arithmetic, $c{=}(a{+}b)\bmod 97$) the
constrained arm is exempt from weight decay by an exact geometric fact:
weight decay has zero Riemannian gradient on $\St(d,r)$, since its gradient is
purely normal (Proposition~\ref{prop:wdecay}:
$\nabla^{\mathrm{euc}}\tfrac{\lambda}{2}\|W\|_F^2 = \Pi_W(\lambda W) = 0$).
Two five-seed studies---default settings (A2 $0.78$, B2 $0.57$, $p{=}0.18$,
not significant) and the \citet{power2022} stabilization (A2 $0.95$, B2
$0.82$, $p{=}0.019$, A2 wins)---confirm the constrained arm does not grok more
reliably than the baseline. The weight-decay exemption
(Proposition~\ref{prop:wdecay}) is an exact structural fact whose grokking
consequence is nil.

A single-seed pilot study (Section~\ref{sec:harnessing}) turns the localization
result into a design test.
Removing weight decay from~$\WQ,\WK$ (H1 without the Stiefel constraint)
allows the routing-frame isometry error to drift to~$18$ while the circuit
regularization remains intact; adding the Stiefel constraint (B2-hi) keeps
the isometry error at~$10^{-5}$ through every slingshot collapse---a
$1.6\times10^6$-factor difference---and produces the first stable grokking
(sustained held-out accuracy $\ge 0.95$ for the final 10\,000 epochs) under
these conditions.
The constraint is not a passive bookkeeping device; it anchors the routing
frame geometrically while the optimizer compresses the circuit.

The method targets tasks where attention geometry is the bottleneck; settings
with frozen pre-trained representations (for example, BERT fine-tuning that
touches only attention heads) fall outside this regime.

As \citet{coates2011} showed for patches on the sphere, the geometric prior can
dominate the algorithm choice. We find the same for attention projections on the
Stiefel manifold---with the further benefit that here the prior comes with a
proof of why it works and an ablation identifying which part does.

\acks{We acknowledge financial support and computational resources provided by
NeuroTechNet S.A.S.}

\bibliographystyle{plainnat}
\bibliography{stiefel_attention}

\end{document}